\documentclass[3p, sort&compress]{elsarticle}
\usepackage{amssymb}
\usepackage{amsmath}
\usepackage{graphicx}
\usepackage{booktabs}
\usepackage[vlined,ruled,linesnumbered]{algorithm2e}
\usepackage{multirow}
\usepackage{subcaption}
\usepackage{adjustbox}
\usepackage{enumitem}
\usepackage{mathrsfs}
\usepackage[colorlinks=true]{hyperref}

\newtheorem{thm}{Theorem}
\newdefinition{rmk}{Remark}
\newproof{pf}{Proof}

\begin{document}
	\begin{frontmatter}		
				\title{Orthogonal Subspace Clustering: Enhancing High-Dimensional Data Analysis through Adaptive Dimensionality Reduction and Efficient Clustering}
				\author{Qing-Yuan Wen}
				\author{Da-Qing Zhang\corref{cor1}}
				\ead{d.q.zhang@ustl.edu.cn}
				\cortext[cor1]{Corresponding author}	
				\address{School of Science, University of Science and Technology Liaoning,
						Anshan, Liaoning Province, 114051, PR China}		
		\begin{abstract}
			This paper presents Orthogonal Subspace Clustering (OSC), an innovative method for high-dimensional data clustering. We first establish a theoretical theorem proving that high-dimensional data can be decomposed into orthogonal subspaces in a statistical sense, whose form exactly matches the paradigm of Q-type factor analysis. This theorem lays a solid mathematical foundation for dimensionality reduction via matrix decomposition and factor analysis. Based on this theorem, we propose the OSC framework to address the "curse of dimensionality"---a critical challenge that degrades clustering effectiveness due to sample sparsity and ineffective distance metrics. OSC integrates orthogonal subspace construction with classical clustering techniques, introducing a data-driven mechanism to select the subspace dimension based on cumulative variance contribution. This avoids manual selection biases while maximizing the retention of discriminative information. By projecting high-dimensional data into an uncorrelated, low-dimensional orthogonal subspace, OSC significantly improves clustering efficiency, robustness, and accuracy. Extensive experiments on various benchmark datasets demonstrate the effectiveness of OSC, with thorough analysis of evaluation metrics including Cluster Accuracy (ACC), Normalized Mutual Information (NMI), and Adjusted Rand Index (ARI) highlighting its advantages over existing methods.
		\end{abstract}
		
		\begin{keyword}
						Orthogonal subspaces \sep High-dimensional data \sep Cumulative variance \sep K-Means clustering
		\end{keyword}
		
	\end{frontmatter}

\section{Introduction}
\label{sec:intro}

The processing and analysis of high-dimensional data have become central topics in fields such as bioinformatics, machine learning, and signal processing \cite{Desquesnes2017,Nematzadeh2024,Wang2025}. Data types like text, images, and gene sequences often possess tens of thousands to millions of features, leading to significant computational and storage burdens. Moreover, the "curse of dimensionality" due to sample sparsity causes traditional distance metrics to become ineffective, increases sensitivity to noise, and substantially degrades clustering performance. Therefore, employing effective dimensionality reduction techniques as a preprocessing step has become essential for enhancing clustering outcomes and uncovering latent structures in high-dimensional data.

Existing dimensionality reduction methods can be broadly categorized into feature selection and
feature extraction. Feature selection reduces dimensionality by evaluating feature importance and
selecting a discriminative subset, preserving the original semantics of the data. Techniques include
weight allocation \cite{Zhao2025} or sparsity constraints \cite{Guo2020}, which have been
successfully applied in clustering cancer microarray data by combining statistical measures with
clustering algorithms \cite{Babu2021}. Feature extraction, on the other hand, constructs new feature representations to capture deep-seated patterns in the data. Examples include clustering ensemble methods based on random projection and fuzzy C-means \cite{Rathore2018}, supervised dimensionality reduction models that jointly learn projection matrices and neighborhood structures \cite{Pang2019}, and the P-SFCM approach integrating sparse fuzzy clustering with principal component analysis \cite{Chen2023}. Additionally, methods such as manifold learning (e.g., t-SNE, UMAP), neural network-based dimensionality reduction (e.g., autoencoders), dictionary learning, generative models (e.g., VAEs), and tensor decomposition offer diverse tools for visualization and representation learning, addressing high-dimensional nonlinearity and structural complexity from different perspectives.

Among these paradigms, subspace clustering methods have attracted sustained attention due to their ability to explicitly capture local structures within high-dimensional data. Unlike traditional clustering, subspace clustering assumes that data inherently reside in a union of low-dimensional subspaces. It improves clustering performance by learning local sparse representations or low-dimensional embeddings to construct affinity graphs \cite{Elhamifar2013,Liu2013,Du2023}. These methods are particularly suitable for data with complex local structures or feature heterogeneity, such as image sets, document collections, and multi-view data. Their advantages include adaptability to the sparse nature of high-dimensional data by building sparse similarity matrices that suppress redundancy and noise; preservation of local geometric structures, with some methods balancing sparsity and connectivity in similarity graphs through criteria like "good neighbors" \cite{Yang2020}; and effective mitigation of metric distortion caused by the curse of dimensionality by performing distance calculations in low-dimensional subspaces. However, traditional subspace clustering still faces two major challenges: first, the construction of similarity graphs and feature decomposition often involve high computational complexity, limiting scalability to large-scale datasets; second, subspace dimensions or parameters frequently rely on empirical settings, lacking data-adaptive mechanisms.

This study makes two key contributions to high-dimensional data clustering:

\begin{enumerate}
	\item \textbf{Theoretical Foundation of Orthogonal Subspace Decomposition} \\
	We establish a rigorous theorem that proves high-dimensional data can be statistically decomposed into orthogonal subspaces, whose structural form is fully consistent with the paradigm of Q-type factor analysis. This theorem provides a solid mathematical basis for dimensionality reduction via matrix decomposition, validating the theoretical rationality of subsequent subspace-based operations.
	
	\item \textbf{Data-Driven Orthogonal Subspace Clustering (OSC) Framework} \\
	We propose an efficient OSC framework integrating orthogonal subspace construction with classical clustering techniques:
	\begin{itemize}
		\item A data-driven dimension selection mechanism that automatically determines the optimal subspace dimension based on cumulative variance contribution, eliminating manual selection bias and maximizing the retention of discriminative information.
		\item A hybrid clustering pipeline that first projects high-dimensional data into an uncorrelated, low-dimensional orthogonal subspace, then performs clustering in this structured space. This approach preserves the structural characteristics of subspace clustering while significantly improving computational efficiency, robustness, and clustering accuracy, providing a scalable and automated solution for high-dimensional data analysis.
	\end{itemize}
\end{enumerate}

The remainder of this paper is structured as follows. Section 2 introduces the theoretical background and mathematical notation essential for understanding the OSC algorithm. Section 3 details the proposed OSC approach, including the correlation matrix decomposition process, cumulative variance evaluation, and the orthogonal subspace projection mechanism. Section 4 presents and discusses the experimental results, providing a comprehensive comparison of OSC with existing methods using various benchmark datasets. The evaluation metrics, including Cluster Accuracy (ACC), Normalized Mutual Information (NMI), and Adjusted Rand Index (ARI), are thoroughly analyzed to highlight the advantages of the OSC method. Finally, Section 5 concludes the paper by summarizing the key findings, discussing the limitations of the OSC method, and outlining directions for future research.

\section{Related Work}

K-Means, as a classic partitioning clustering algorithm, has wide applications in data mining and pattern recognition. The algorithm achieves data partitioning through iterative optimization of the sum of squared distances from samples to cluster centers. Its objective function is defined as:
\[ \min_{C_i} \sum_{i=1}^k \sum_{x \in C_i} \|x - \mu_i\|^2 \]
where \( C_i \) represents the \( i \)-th cluster, and \( \mu_i = \frac{1}{|C_i|} \sum_{x \in C_i} x \) is the centroid of the corresponding cluster. The algorithm's iterative process consists of two alternating steps: first, each sample is assigned to the nearest centroid's cluster, and second, the centroid position is updated based on the samples within the cluster, until the objective function converges or the maximum number of iterations is reached.

The tendency of Euclidean distance to converge in high-dimensional spaces is a core challenge for K-Means. As the dimensionality increases, the variance of distances between samples gradually decreases, causing all samples to become uniformly distributed in the distance space. This phenomenon renders the clustering strategy based on nearest centroid assignment ineffective, as the differences in distances from samples to different centroids no longer provide sufficient discriminative power. From an information theory perspective, data points in high-dimensional spaces tend to be distributed near the vertices of a hypercube, with Euclidean distance primarily reflecting the number of dimensions rather than the intrinsic structure of the data. Additionally, the curse of dimensionality exacerbates data sparsity, making the boundaries of each cluster indistinct. In typical high-dimensional applications such as gene expression data in bioinformatics and term frequency vectors in text mining, directly applying K-Means to the original space often yields suboptimal or even incorrect clustering results.

Matrix decomposition provides a solid mathematical foundation for dimensionality reduction of high-dimensional data. The singular value decomposition of a data matrix \( X \in \mathbb{R}^{N \times p} \) is represented as \( X = U \Sigma V^\top \), where \( U \in \mathbb{R}^{N \times N} \) and \( V \in \mathbb{R}^{p \times p} \) are orthogonal matrices, and \( \Sigma \in \mathbb{R}^{N \times p} \) is a diagonal matrix containing the singular values. This decomposition maps the data to a sample space spanned by the left singular vectors and a feature space spanned by the right singular vectors, essentially forming a geometric interpretation of second-order tensor decomposition. The core advantage of singular value decomposition lies in its optimal approximation property: for any rank \( r < \text{rank}(X) \), the truncated singular value decomposition \( U_r \Sigma_r V_r^\top \) is the best rank-\( r \) approximation of \( X \) in the Frobenius norm sense. This property provides a theoretical basis for dimensionality reduction, where the most data information is retained in the sense of minimal reconstruction error by preserving the subspace corresponding to the largest singular values.

Singular value decomposition is deeply connected to factor analysis, which models observed data by introducing latent variables. The factor analysis model assumes that the original data can be decomposed into a linear combination of common factors and specific factors:
\[ X - \mu = F \Lambda^\top + \varepsilon \]
where \( F \in \mathbb{R}^{N \times m} \) is the common factor matrix, with its column vectors forming a set of orthogonal variables, i.e., \( F^\top F = I_m \); \( \Lambda \in \mathbb{R}^{p \times m} \) is the factor loading matrix, describing the association strength between observed variables and common factors; and \( \varepsilon \in \mathbb{R}^{N \times p} \) is the specific factor matrix, representing the unique variance component of each observed variable. The model assumes that specific factors are mutually orthogonal, i.e., \( \varepsilon \varepsilon^\top = \Psi \) is a diagonal matrix; and common factors are orthogonal to specific factors, i.e., \( F^\top \varepsilon = 0 \). These orthogonality assumptions not only simplify the model estimation process but also provide a clear statistical interpretation: common factors reflect the shared variance among multiple observed variables, while specific factors represent the unique information of each variable.

Estimating the factor loading matrix \( \Lambda \) is the core problem of factor analysis. Traditional methods include the principal axis factoring method, maximum likelihood estimation, and iterative principal axis factoring. Among these, maximum likelihood estimation has asymptotic optimality under the normality assumption, while the principal axis factoring method is more suitable for non-normal data.

Although both factor analysis and principal component analysis are dimensionality reduction methods, they have fundamental differences. Principal component analysis is a purely geometric transformation aimed at maximizing the variance of the data after projection, with components strictly orthogonal but without statistical model assumptions. In contrast, factor analysis is a statistical modeling method that explicitly assumes the existence of latent factors driving the variation in observed variables. In terms of estimation objectives, principal component analysis directly obtains component directions through singular value decomposition, while factor analysis requires iterative estimation of the loading matrix and specific variances. In terms of interpretability, the factor loadings of factor analysis have clear variable association meanings, whereas principal components merely represent the directions of maximum variance in the data. In practical applications, when researchers focus on the association structure between variables, factor analysis is more appropriate; when the goal is simply data compression or visualization, principal component analysis is more direct and efficient. This equivalence reveals the deep intrinsic connection between matrix decomposition and statistical modeling, laying a solid theoretical foundation for dimensionality reduction methods based on sample correlation structures.

\section{Methodology}

To address the limitations of traditional subspace clustering methods and enhance their applicability to high-dimensional data, we propose the Orthogonal Subspace Clustering (OSC) framework. This section details the theoretical foundations and methodological innovations of OSC, including the correlation matrix decomposition process, cumulative variance evaluation, and the orthogonal subspace projection mechanism.

Consider the data matrix \( X = (x_{ij}) \in \mathbb{R}^{N \times p} \), where \( N \) represents the sample size and \( p \) denotes the number of attributes. The matrix \( X \) encapsulates the high-dimensional data that we aim to cluster. To standardize the data and eliminate the mean effect, we define the centered data matrix \( X_c^\top = X^\top - \boldsymbol{1}_p \mu^\top \), where \( \boldsymbol{1}_p \) is a column vector of ones with length \( p \), and \( \mu \) is the mean of each row of \( X \). The mean \( \mu \) is calculated as:
\[ \mu = (\mu_1, \mu_2, \dots, \mu_N)^\top \]
where
\[ \mu_k = \frac{1}{p} \sum_{j=1}^{p} x_{k,j} \]
for each row \( k \). This centering process ensures that the subsequent analyses are based on the deviations of the data points from their respective means, thereby enhancing the robustness of the clustering algorithm.

Let \( Y = X_c^\top \), which represents the transposed and centered data matrix. The theoretical foundation for OSC is established by Theorem 1, which describes the conditions under which the observed data matrix \( Y \) can be decomposed into a union of low-dimensional subspaces. This theorem provides the necessary mathematical framework to support the OSC method's effectiveness in capturing the underlying structure of high-dimensional data.

\begin{thm}[Multi-Distribution Heteroscedastic Subspace Statistical Orthogonality]
~
	
Let the observation data matrix \({Y} = [{y}_1, \dots, {y}_N] \in \mathbb{R}^{p \times N}\) have column vectors satisfying the following conditions: the samples are partitioned into \(k\) disjoint clusters \(C_1, \dots, C_k\) with \(|C_i| = N_i\) and \(\sum_{i=1}^k N_i = N\). For each cluster \(C_i\), there exists a fixed \(m_i\)-dimensional linear subspace \(\mathcal{S}_i \subset \mathbb{R}^p\) such that when \(j \in C_i\), \({y}_j = {z}_j + {e}_j\), where the signal component \({z}_j \in \mathcal{S}_i\) and the noise component \({e}_j \sim \mathcal{N}({0}, \sigma_i^2 {I}_p)\). Assume that the signal vectors \(\{{z}_j\}_{j \in C_i}\) are independently and identically distributed on \(\mathcal{S}_i\) with covariance matrices having minimum eigenvalues \(\lambda_{\min} \geq c_i > 0\); the noise vectors \(\{{e}_j\}\) are mutually independent and independent of the signals; the subspaces are separated with \(\min_{i \neq j} \|{P}_i - {P}_j\|_2 \geq \delta > 0\), where \({P}_i\) is the orthogonal projection matrix of \(\mathcal{S}_i\); and the noise variances are bounded with \(\max_i \sigma_i^2 \leq \sigma_{\max}^2 < \infty\) and the noise energy satisfies \(\|{E}\|_F \ll \|{Z}\|_F\).

Let \(m_{\text{union}} = \dim(\operatorname{span}(\cup_{i=1}^k \mathcal{S}_i))\) and take \(m \geq m_{\text{union}}\). Let \(\widehat{{U}}_m\) be the matrix formed by the first \(m\) left singular vectors of \({Y}\) in its singular value decomposition, and define the residual matrix \(\widehat{{\varepsilon}} = {Y} - \widehat{{U}}_m \widehat{{U}}_m^\top {Y}\). Then:
\[
\widehat{{U}}_m^\top \widehat{{U}}_m = {I}_m, \quad \widehat{{U}}_m^\top \widehat{{\varepsilon}} = {0}_{m \times N},
\]
and the residual covariance matrix satisfies
\[
\mathbb{E}[\widehat{{\varepsilon}}^\top \widehat{{\varepsilon}}] = \operatorname{block-diag}\left( \sigma_1^2(p-m_1){I}_{N_1}, \dots, \sigma_k^2(p-m_k){I}_{N_k} \right) + {\Delta}_N,
\]
where the error matrix \({\Delta}_N\) satisfies: \({\Delta}_N(a,b) = 0\) when samples \(a,b\) belong to different clusters; \(|{\Delta}_N(a,b)| \leq C\sigma_{\max}^2/(\delta \sqrt{N})\) when \(a \neq b\) belong to the same cluster \(C_i\); and its Frobenius norm satisfies \(\|{\Delta}_N\|_F = \mathcal{O}(\sigma_{\max}^2/(\delta \sqrt{N}))\).

\end{thm}

\begin{pf}

Let \({Y} = {U}{\Sigma}{V}^\top\) be the compact singular value decomposition, where \({U} \in \mathbb{R}^{p \times r}\) has orthonormal columns and \(r = \operatorname{rank}({Y})\). By \(\widehat{{U}}_m = {U}(:,1:m)\) and the orthonormality of \({U}\), we directly obtain \(\widehat{{U}}_m^\top \widehat{{U}}_m = {I}_m\). By the definition of the residual:
\[
\widehat{{U}}_m^\top \widehat{{\varepsilon}} = \widehat{{U}}_m^\top ({Y} - \widehat{{U}}_m \widehat{{U}}_m^\top {Y}) = {0},
\]
which proves the first two conclusions.

To prove the third conclusion, let \(\mathcal{S}_{\text{union}} = \operatorname{span}(\cup_{i=1}^k \mathcal{S}_i)\), \({U}_{\text{union}}\) be its orthonormal basis, and \({P}_{\text{union}} = {U}_{\text{union}}{U}_{\text{union}}^\top\) be the projection matrix. 

{\bf In the ideal case}, the residual is:
\[
{\varepsilon} = ({I}_p - {P}_{\text{union}}){E},
\]
since \(\mathcal{S}_i \subseteq \mathcal{S}_{\text{union}}\) implies \({P}_{\text{union}}{Z} = {Z}\). By the projection property, for \(j \in C_i\) we have \(({I}_p - {P}_{\text{union}}){e}_j = ({I}_p - {P}_i){e}_j\). Calculating the expectations yields:

\begin{itemize}
	\item  For same-cluster samples: \(\mathbb{E}[\|{\varepsilon}_{\cdot j}\|_2^2] = \sigma_i^2(p - m_i)\)
	
	\item For different samples within the same cluster: \(\mathbb{E}[{\varepsilon}_{\cdot j}^\top {\varepsilon}_{\cdot l}] = 0\)
	
	\item  For samples from different clusters: \(\mathbb{E}[{\varepsilon}_{\cdot j}^\top {\varepsilon}_{\cdot l}] = 0\)
\end{itemize}

Thus, \(\mathbb{E}[{\varepsilon}^\top {\varepsilon}] = \operatorname{block-diag}\left( \sigma_1^2(p-m_1){I}_{N_1}, \dots, \sigma_k^2(p-m_k){I}_{N_k} \right)\).

{\bf In the actual case}, the residual decomposes as:
\[
\widehat{{\varepsilon}} = ({I}_p - {P}_{\text{union}}){E} + ({P}_{\text{union}} - \widehat{{P}}_m){Z} + ({P}_{\text{union}} - \widehat{{P}}_m){E},
\]
where \(\widehat{{P}}_m = \widehat{{U}}_m \widehat{{U}}_m^\top\). By the Davis-Kahan theorem:
\[
\|\widehat{{P}}_m - {P}_{\text{union}}\|_F = \mathcal{O}_p(\sigma_{\max}/(\delta \sqrt{N})).
\]
The residual covariance expectation is:
\[
\mathbb{E}[\widehat{{\varepsilon}}^\top \widehat{{\varepsilon}}] = \mathbb{E}[{\varepsilon}^\top {\varepsilon}] + \mathbb{E}[{\varepsilon}^\top {\Delta} + {\Delta}^\top {\varepsilon}] + \mathbb{E}[{\Delta}^\top {\Delta}],
\]
where \({\Delta} = ({P}_{\text{union}} - \widehat{{P}}_m)({Z} + {E})\). By the Cauchy-Schwarz inequality:
\[
\|\mathbb{E}[{\varepsilon}^\top {\Delta}]\|_F = \mathcal{O}(\sigma_{\max}^2/(\delta \sqrt{N})), \quad \mathbb{E}[\|{\Delta}^\top {\Delta}\|_F] = \mathcal{O}(\sigma_{\max}^2/(\delta^2 N)).
\]
The expectations of residuals from different clusters are zero, and the off-diagonal elements within the same cluster are \(\mathcal{O}(\sigma_{\max}^2/(\delta \sqrt{N}))\), ensuring the error matrix \({\Delta}_N\) meets the theorem's requirements. \qed
\end{pf}

Theorem 1 establishes that the singular value decomposition of the transposed centered data matrix ${Y} = {X}_c^{\top} {D}^{-1}$ yields a residual structure consistent with the factor analysis model in Section 2. Assuming ${Y}$ satisfies the conditions of Theorem 1, this theoretical alignment motivates our adoption of factor analysis for dimensionality reduction within the Orthogonal Subspace Clustering (OSC) framework. The factor model is expressed as:
\[
{Y} = {X}_c^{\top} {D}^{-1} = {F} {A}_m^{\top} + {\varepsilon}
\]
where ${F} \in \mathbb{R}^{p \times m}$ denotes the factor score matrix, ${A}_m \in \mathbb{R}^{N \times m}$ the factor loading matrix, and ${\varepsilon}$ the specific factor matrix. We impose the following orthogonality assumptions on the model:
1. The columns of ${F}$ are mutually orthogonal unit vectors, i.e., ${F}^\top {F} = {I}_m$.
2. The factor scores ${F}$ are orthogonal to the specific factors ${\varepsilon}$, i.e., ${F}^\top {\varepsilon} = {\boldsymbol{0}}$.
3. The specific factors ${\varepsilon}$ have uncorrelated columns, i.e., ${\varepsilon}^\top {\varepsilon}$ is diagonal.

The focus of this model is the calculation of the factor loading matrix ${A}_m$.  We need to address two key problems: 1) estimating the parameter $m$, and 2) estimating the factor loading matrix itself. As mentioned in Section 2, there are three methods for calculating the loading matrix. Here, we adopt the principal axis factoring method. The steps are as follows:

1. Calculate the sample correlation matrix ${R}_{\text{samples}} = {D}^{-1} {X}_c {X}_c^{\top} {D}^{-1}$ and perform its spectral decomposition:
\[
{R}_{\text{samples}} = {U} {\Lambda} {U}^{\top}
\]
where ${U}$ is the matrix of orthogonal eigenvectors and ${\Lambda}$ is the diagonal matrix of ordered eigenvalues.

2. Determine the number of factors $m$ using the cumulative variance contribution rate:
\[
\theta(m) = \frac{\sum_{i=1}^m \lambda_i}{\sum_{i=1}^N \lambda_i}
\]
The smallest $m$ satisfying $\theta(m) \geq \theta_0$ (typically 0.85 or 0.90) is selected, ensuring sufficient discriminative information retention while avoiding manual parameter setting.

3. Express the sample correlation matrix in the form:
\[
{R}_{\text{samples}} = ({F} {A}_m^{\top} + {\varepsilon})^{\top} ({F} {A}_m^{\top} + {\varepsilon})
\]
Expanding this expression and utilizing the orthogonality assumptions, we obtain:
\[
{R}_{\text{samples}} = {A}_m {A}_m^{\top} + {\Psi}
\]
where ${A}_m = {U}({:,1:m}) {\Lambda}_{1:m}^{1/2}$ and ${\Psi}$ is the diagonal matrix of specific variances. When the influence of specific factors is negligible (${\varepsilon} \approx {0}$), we have:
\[
{R}_{\text{samples}} \approx {A}_m {A}_m^{\top}
\]
This approximation allows us to estimate the factor loading matrix ${A}_m$ from the spectral decomposition of ${R}_{\text{samples}}$.

Having obtained the loading matrix, we can express ${X}_c^{\top}$ as:
\[
{X}_c^{\top} = {F} {A}_m^{\top} {D}
\]
This equation demonstrates that the column vectors of ${X}_c^{\top}$ (the centered samples) can be approximately represented as linear combinations of the orthogonal subspace basis vectors ${F}$. The factor loading matrix ${D} {A}_m \in \mathbb{R}^{N \times m}$ contains the coordinates of each centered sample in the orthogonal subspace $\text{span}\{{F}\}$.

This completes the dimensionality reduction process, mapping the original $p$-dimensional samples to $m$-dimensional coordinate vectors. The transformation effectively projects high-dimensional data into a lower-dimensional orthogonal subspace while preserving essential structural information. This projection not only reduces computational complexity but also enhances the interpretability and effectiveness of subsequent analyses, particularly in clustering applications.

The K-Means clustering algorithm will then be applied to the matrix ${D} {A}_m$, which represents the samples in the reduced $m$-dimensional orthogonal subspace. This application achieves the clustering objective by grouping similar samples together in the lower-dimensional space. The use of the orthogonal subspace ensures that the distance metrics used in K-Means are more reliable and meaningful, thereby improving the overall clustering performance.

Algorithm 1  outlines the steps for estimating the number of factors \( m \) and the factor loading matrix \( A_m \). This algorithm provides a systematic approach to dimensionality reduction, ensuring that the underlying structure of the high-dimensional data is preserved.

\begin{algorithm}[htbp]
	\DontPrintSemicolon
	\SetAlFnt{\footnotesize\ttfamily} 
	\KwIn{Raw data matrix ${X} = ({x}_{ij}) \in \mathbb{R}^{N \times p}$; Threshold for cumulative variance contribution rate $\theta_0$ (typically between 0.7 and 0.9)}
	\KwOut{Cluster assignments for the samples}
	\tcp{Initialize cluster assignments}
	$\text{cluster\_assignments} \leftarrow \text{None}$\;
	
	\tcp{Data Preprocessing}

	\For{$k \leftarrow 1$ \KwTo $N$}{
		$\mu_k \leftarrow \frac{1}{p} \sum_{j=1}^p x_{kj}$\;
	}
	
	${\mu} \leftarrow (\mu_1, \mu_2, \dots, \mu_N)^\top \in \mathbb{R}^N$\;

	\For{$k \leftarrow 1$ \KwTo $N$}{
		$\sigma_k \leftarrow \sqrt{\frac{1}{p-1} \sum_{j=1}^p (x_{kj} - \mu_k)^2}$\;
	}
	
	${D} \leftarrow \text{diag}(\sigma_1, \sigma_2, \dots, \sigma_N) \in \mathbb{R}^{N \times N}$\;
	
	${X}_c^\top \leftarrow {X}^\top - {1}_N {\mu}^\top$\;
	
	${R}_{\text{samples}} \leftarrow {D}^{-1} {X}_c {X}_c^\top {D}^{-1}$\;
	
	\tcp{Spectral Decomposition}
	${R}_{\text{samples}} \leftarrow {U} {\Lambda} {U}^\top$\;
	
	\tcp{Determine Number of Factors}
	\For{$m \leftarrow 1$ \KwTo $N$}{
		$\theta(m) \leftarrow \frac{\sum_{i=1}^m \lambda_i}{\sum_{i=1}^N \lambda_i}$\;
		\If{$\theta(m) \geq \theta_0$}{
			\Return{$m$}\;
		}
	}
	
	\tcp{Estimate Factor Loading Matrix}
	${A}_m \leftarrow {U}({:,1:m}) {\Lambda}_{1:m}^{1/2}$\;
	\tcp{Clustering}
	$\text{cluster\_assignments} \leftarrow \text{K-Means}({D} {A}_m)$\;
	\Return{$\text{cluster\_assignments}$}\;
	\caption{Orthogonal Subspace Clustering with Factor Analysis}
\end{algorithm}

The OSC method leverages the structural properties of the underlying subspaces to enable efficient and accurate clustering of high-dimensional data. By integrating orthogonal subspace construction with classical clustering techniques, OSC addresses the challenges of high computational complexity and the need for data-adaptive mechanisms. The method's innovations include a data-driven mechanism for selecting the dimension of the orthogonal subspace and a clustering framework that fuses orthogonal subspace projection with K-Means. These advancements ensure computational efficiency and significantly improve clustering robustness and accuracy.

	\section{Experiments and Results Analysis}
\label{sec:experiment}

\subsection{Experimental Setup}
\subsubsection{Dataset}
This paper evaluates the proposed method using nine real high-dimensional datasets, including the image datasets ORL~\cite{Samaria1994}, Jaffe~\cite{Lyons1999}, MNIST~\cite{LeCun1998}, Coil20~\cite{Nene1996}, Yale~\cite{He2005}, CASIA-FaceV5~\cite{Ma2022}, and biological gene datasets bladderGSE89~\cite{Mengual2009}, gastricGSE2685~\cite{Yoshitaka2002}, and MLL~\cite{Golub1999}. Table \ref{tab:dataset_info} provides basic statistical information for the nine datasets, whereas Figure \ref{fig:datasets} shows representative samples from the six picture datasets. 2000 samples from the MNIST dataset and 4500 samples from 100 categories in the CASIA-FaceV5 dataset were randomly picked in order to manage computational complexity for comparative tests.
\begin{figure}[htbp]
	\centering
	\includegraphics[width=0.15\textwidth, height=0.08\textheight]{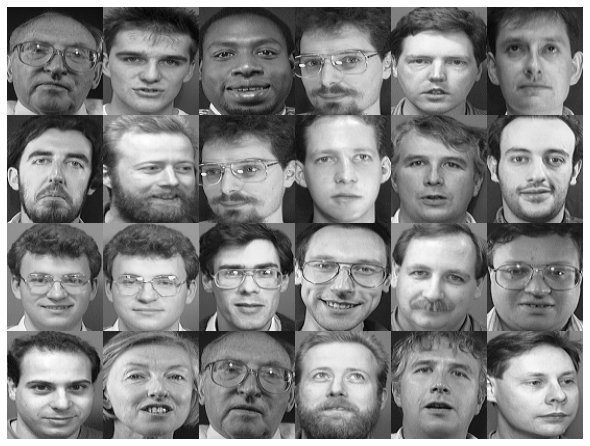} \hfill
	\includegraphics[width=0.15\textwidth, height=0.08\textheight]{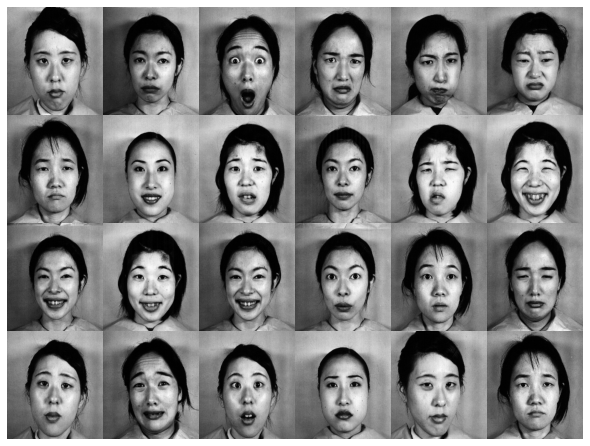} \hfill
	\includegraphics[width=0.15\textwidth, height=0.08\textheight]{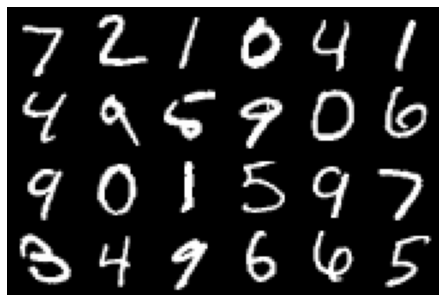} \hfill
	\includegraphics[width=0.15\textwidth, height=0.08\textheight]{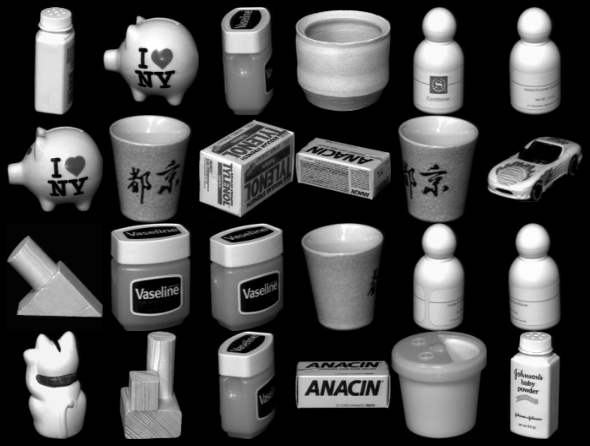} \hfill
	\includegraphics[width=0.15\textwidth, height=0.08\textheight]{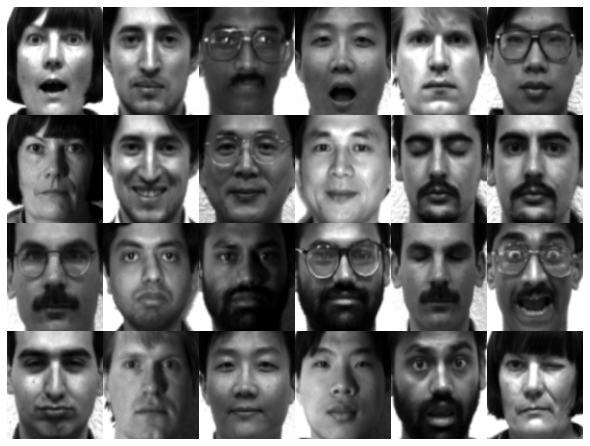} \hfill
	\includegraphics[width=0.15\textwidth, height=0.08\textheight]{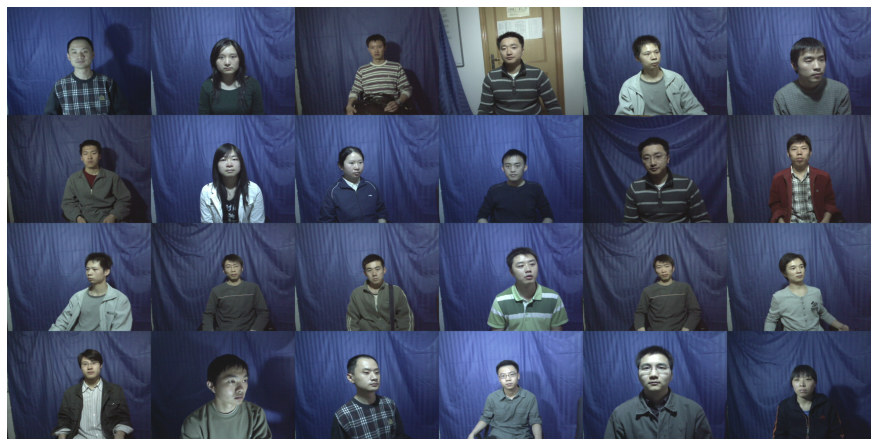} \\
	\caption{Representative samples from six image datasets (from left to right): ORL, Jaffe, MNIST, Coil20, Yale, and CASIA-FaceV5}
	\label{fig:datasets}
\end{figure}
\begin{table}[htbp]
	\small
	\centering
	\caption{Real-world datasets description}
	\label{tab:dataset_info}
	\adjustbox{max width=\textwidth}{%
		\begin{tabular}{lcccc}
			\toprule
			Data & Sample Size &  Features &  True Categories &  Source \\
			\midrule
			ORL & 400 & 92×112 & 40 & Cambridge University AT\&T Lab \\
			Jaffe & 212 & 256×256 & 10 & JAFFE Face Database \\
			MNIST & 2000 & 28×28 & 10 & LeCun MNIST Database \\
			Coil20 & 1440 & 128×128 & 20 & Columbia University COIL-20 \\
			Yale & 165 & 320×243 & 15 & Yale Face Database \\
			CASIA-FaceV5 & 4500 & 640×480 & 100 & CASIA Face Database \\
			bladderGSE89 & 40 & 5724 & 3 & GEO: GSE89 \\
			gastricGSE2685 & 30 & 4522 & 2 & GEO: GSE2685 \\
			MLL & 72 & 12533 & 3 & UCI Machine Learning Repository \\
			\bottomrule
		\end{tabular}%
	}
\end{table}

\subsubsection{Cluster evaluation metrics}
We used three metrics---Normalized Mutual Information (NMI), Cluster Accuracy (ACC), and Adjusted Rand Index (ARI)---to assess the clustering quality produced by each technique.

After determining the best mapping link between cluster labels and true labels, ACC computes the percentage of correctly classified samples. It gauges how closely clustering results match actual labels.
\begin{equation}
	\mathit{ACC} = \frac{\sum_{i=1}^n \delta(r_i; \mathit{map}(c_i))}{n}, \quad \delta(x, y) = 
	\begin{cases} 
		1 & \mathit{if}\; x = y \\
		0 & \mathit{otherwise}
	\end{cases}
	\label{eq:acc}
\end{equation}

In this case, $r_i$ and $c_i$ represent the true label and the obtained label for the data point ${x_i}$, respectively. The Hungarian algorithm's optimal redistribution is represented by $\text{map}$.

The correlation between data distributions is measured via mutual information. Stronger correlation is indicated by higher mutual information. The joint distribution of random samples $(x,y)$ is $p(x,y)$, whereas the marginal distributions of samples $x$ and $y$, respectively, are $p(x)$ and $p(y)$. The relative entropy between the joint distribution $p(x,y)$ and the marginal distributions $p(x)p(y)$ is represented by the mutual information $R(X;Y)$. Normalization is used to obtain the Normalized Mutual Information (NMI), which varies from $[0,1]$, in order to reduce bias brought on by information inflation. Its basic idea is to use information entropy to measure how much information about one variable is contained in another.

\begin{equation}
	R(X; Y) = \sum_{x \in X} \sum_{y \in Y} p(x, y) \log \frac{p(x, y)}{p(x) p(y)}
	\label{eq:mutual-info}
\end{equation}

\begin{equation}
	\mathit{NMI} = \frac{R(X; Y)}{\sqrt{H(X) H(Y)}}
	\label{eq:nmi}
\end{equation}
In this case, information entropy is represented by $H(X) = -\sum_{i} p(x_{i}) \log p(x_{i})$ and $H(Y) = -\sum_{j} p(y_{j}) \log p(y_{j})$.

A statistic called ARI is used to assess how closely clustering results match actual labels. It provides a more trustworthy evaluation of clustering model performance by addressing the problem of the Rand coefficient AR scoring unduly high in random partitions by introducing an expectation correction for random clustering.
\begin{equation}
	\mathit{ARI} = \frac{\mathit{RI} - E[\mathit{RI}]}{\max(\mathit{RI}) - E[\mathit{RI}]}
\end{equation}
The Rand coefficient is $\mathit{RI} = \frac{a + d}{a + b + c + d}$. The number of sample pairs that fall into the same category in both the clustered labels $G$ and the true labels $T$ is indicated by $a$. The number of samples that fall into the same category in $T$ but different categories in $G$ is indicated by $b$, and the number of samples that fall into different categories in $T$ but the same category in $G$ is shown by $c$. The number of samples in $T$ and $G$ that fall into distinct categories is represented by $d$.
\subsection{Experimental Results}
This paper compares the performance of the proposed OSC algorithm with four existing approaches: Sparse Subspace Clustering (SSC)\cite{Elhamifar2013}, Low-Rank Representation (LRR)\cite{Liu2013}, Factor K-Means (FKM)\cite{Vichi2001}, and Reduced-Dimension K-Means (RKM)\cite{MacQueen1967}.  Nine real-world datasets were used for the experiments, and parameter adjustment was used to guarantee fair comparisons: The optimal parameter values for the SSC algorithm were selected from the set $\{5,\allowbreak 10,\allowbreak 20,\allowbreak 50,\allowbreak 80,\allowbreak 100,\allowbreak 200,\allowbreak 500,\allowbreak 800\}$. The LRR algorithm was tuned within the range $\{0.001,\allowbreak 0.01,\allowbreak 0.02,\allowbreak 0.05,\allowbreak 0.1,\allowbreak 0.2,\allowbreak 1,\allowbreak 2,\allowbreak 5\}$. FKM and RKM required no hyperparameter tuning. The experimental findings are displayed in tabular form, with columns indicating evaluation metrics for each comparison approach and rows representing various datasets.  The best outcomes are indicated in bold.

\begin{table}[htbp]
	\small
	\centering
	\caption{Algorithm Comparison on ACC}
	\label{tab:acc_results_bold}
	\begin{tabular}{lcccccc}
		\toprule
		Datasets       & SSC     & LRR     & FKM     & RKM     & OSC     \\
		\midrule
		ORL            & 0.817   & 0.705   & 0.573   & 0.608   & \textbf{0.865} \\
		Jaffe          & 0.925   & 0.735   & 0.755   & 0.715   & \textbf{0.940} \\
		MNIST          & \textbf{0.656} & 0.628   & 0.547   & 0.519   & 0.589   \\
		Coil20         & 0.684   & 0.649   & 0.603   & 0.476   & \textbf{0.792} \\
		yale           & 0.775 & 0.721   & 0.570   & 0.624   & \textbf{0.778}   \\
		CASIA-FaceV5   & 0.600   & 0.560   & 0.530   & 0.440   & \textbf{0.684} \\
		bladderGSE89   & 0.550   & 0.600   & 0.675   & 0.625   & \textbf{0.825} \\
		gastricGSE2685 & 0.767   & 0.900   & 0.900   & 0.920   & \textbf{0.963} \\
		MLL            & 0.805   & 0.611   & 0.639   & 0.653   & \textbf{0.847} \\
		\bottomrule
	\end{tabular}
\end{table}

Experimental results from three evaluation metrics—clustering accuracy (ACC), normalized mutual information (NMI), and adjusted kappa (ARI)—show that the proposed method performs better than comparative algorithms on most image datasets (ORL, Jaffe, MNIST, Coil20, Yale, CASIA-FaceV5) and biomedical datasets (bladderGSE89, gastricGSE2685, MLL).   This demonstrates its effectiveness for high-dimensional grouping tasks.

The ACC metric quantifies the consistency between true labels and clustering outcomes.   With an ACC of 0.865 on the ORL dataset, our method performs better than SSC (0.817), LRR (0.705), FKM (0.573), and RKM (0.608), as shown in Table \ref{tab:acc_results_bold}.   With an ACC of 0.940, the proposed method significantly outperforms other algorithms on the Jaffe dataset.  Furthermore, it still has a clear advantage on datasets like Coil20, bladderGSE89, gastricGSE2685, and MLL.   On the MNIST dataset, the recommended method only yields an ACC of 0.589, which is slightly lower than that of other comparison methods.

\begin{table}[htbp]
	\small
	\centering
	\caption{Algorithm Comparison on NMI}
	\label{tab:nmi_results}
	\begin{tabular}{lcccccc}
		\toprule
		Datasets       & SSC     & LRR     & FKM     & RKM     & OSC     \\
		\midrule
		ORL            & 0.906   & 0.824   & 0.800   & 0.817   & \textbf{0.931} \\
		Jaffe          & 0.937   & 0.840   & 0.843   & 0.826   & \textbf{0.944} \\
		MNIST          & \textbf{0.648} & 0.529   & 0.497   & 0.502   & 0.531   \\
		Coil20         & \textbf{0.864} & 0.750   & 0.726   & 0.688   & 0.843   \\
		yale           & 0.785   & 0.758   & 0.690   & 0.735   & \textbf{0.797} \\
		CASIA-FaceV5   & 0.761   & 0.732   & 0.737   & 0.690   & \textbf{0.880} \\
		bladderGSE89   & 0.519   & 0.455   & 0.452   & 0.378   & \textbf{0.646} \\
		gastricGSE2685 & 0.401   & 0.532   & 0.459   & 0.649   & \textbf{0.732} \\
		MLL            & 0.545   & 0.483   & 0.408   & 0.390   & \textbf{0.654} \\
		\bottomrule
	\end{tabular}
\end{table}

NMI is used to evaluate how well clustering results correspond to actual categories at the information level.   The recommended method performs better on the ORL dataset with an NMI value of 0.931 than SSC (0.906) and LRR (0.824), as shown in Table \ref{tab:nmi_results}.   The Jaffe dataset's clustering results, as indicated by the NMI of 0.944, demonstrate a highly consistent information structure with the true categories.   With an NMI of 0.646—much higher than the optimal comparison algorithm SSC (0.519)—our method successfully detected intra-class compactness and inter-class separation in the bladderGSE89 dataset.   Additionally, the proposed method demonstrated improved correlation collecting capabilities for datasets like as CASIA-FaceV5 and gastricGSE2685. Only on the MNIST and Coil20 datasets were the NMI values of the proposed technique slightly lower than those of the comparator methods.

\begin{table}[htbp]
	\small
	\centering
	\caption{Algorithm Comparison on ARI}
	\label{tab:ari_results}
	\begin{tabular}{lcccccc}
		\toprule
		Datasets       & SSC     & LRR     & FKM     & RKM     & OSC     \\
		\midrule
		ORL            & 0.732   & 0.457   & 0.470   & 0.507   & \textbf{0.816} \\
		Jaffe          & 0.962   & 0.672   & 0.664   & 0.640   & \textbf{0.965} \\
		MNIST          & \textbf{0.516} & 0.412   & 0.478   & 0.471   & 0.475   \\
		Coil20         & 0.624   & 0.499   & 0.495   & 0.446   & \textbf{0.722} \\
		yale           & 0.625   & 0.587   & 0.457   & 0.519   & \textbf{0.671} \\
		CASIA-FaceV5   & \textbf{0.771} & 0.452   & 0.452   & 0.459   & 0.552   \\
		bladderGSE89   & 0.353   & 0.405   & 0.398   & 0.341   & \textbf{0.658} \\
		gastricGSE2685 & 0.514   & 0.595   & 0.608   & 0.725   & \textbf{0.891} \\
		MLL            & 0.522   & 0.447   & 0.359   & 0.353   & \textbf{0.541} \\
		\bottomrule
	\end{tabular}
\end{table}

ARI evaluates the degree of consistency between clustering results and true labels while taking random clustering effects into consideration.  The recommended method produces an ARI of 0.816 on the ORL dataset, significantly higher than SSC (0.732) and LRR (0.457), as shown in Table \ref{tab:ari_results}.  The gastric GSE2685 dataset had an ARI of 0.891, which was also a significant improvement over other algorithms; the bladder GSE89 dataset had an ARI of 0.858, while the optimal comparison algorithm LRR only reached 0.405; and the Jaffe dataset had an ARI of 0.658, which was nearly perfect agreement.
\subsection{Experimental Analysis}
\subsubsection{Sensitivity analysis of OSC algorithms' subspace dimensions} 

The cumulative variance explanation rate determines the subspace dimension, or the number of variables.  Cumulative variance explanation rate thresholds of 0.70, 0.75, 0.80, 0.85, and 0.90 were established in order to examine the effect of this dimension on clustering performance.  ACC was used as the assessment metric to analyze clustering efficacy under various thresholds; the findings are displayed in Figure \ref{fig:clustering_acc_large}.  cumulative variance explanation rate criteria, which show the percentage of variance kept after dimensionality reduction in relation to the initial total variance, are represented on the horizontal axis.  Values closer to 1 imply better clustering performance. The vertical axis represents clustering accuracy, which measures the consistency between clustering results and true labels.
\begin{figure}[htbp]
	\centering
	
	\includegraphics[width=0.45\textwidth]{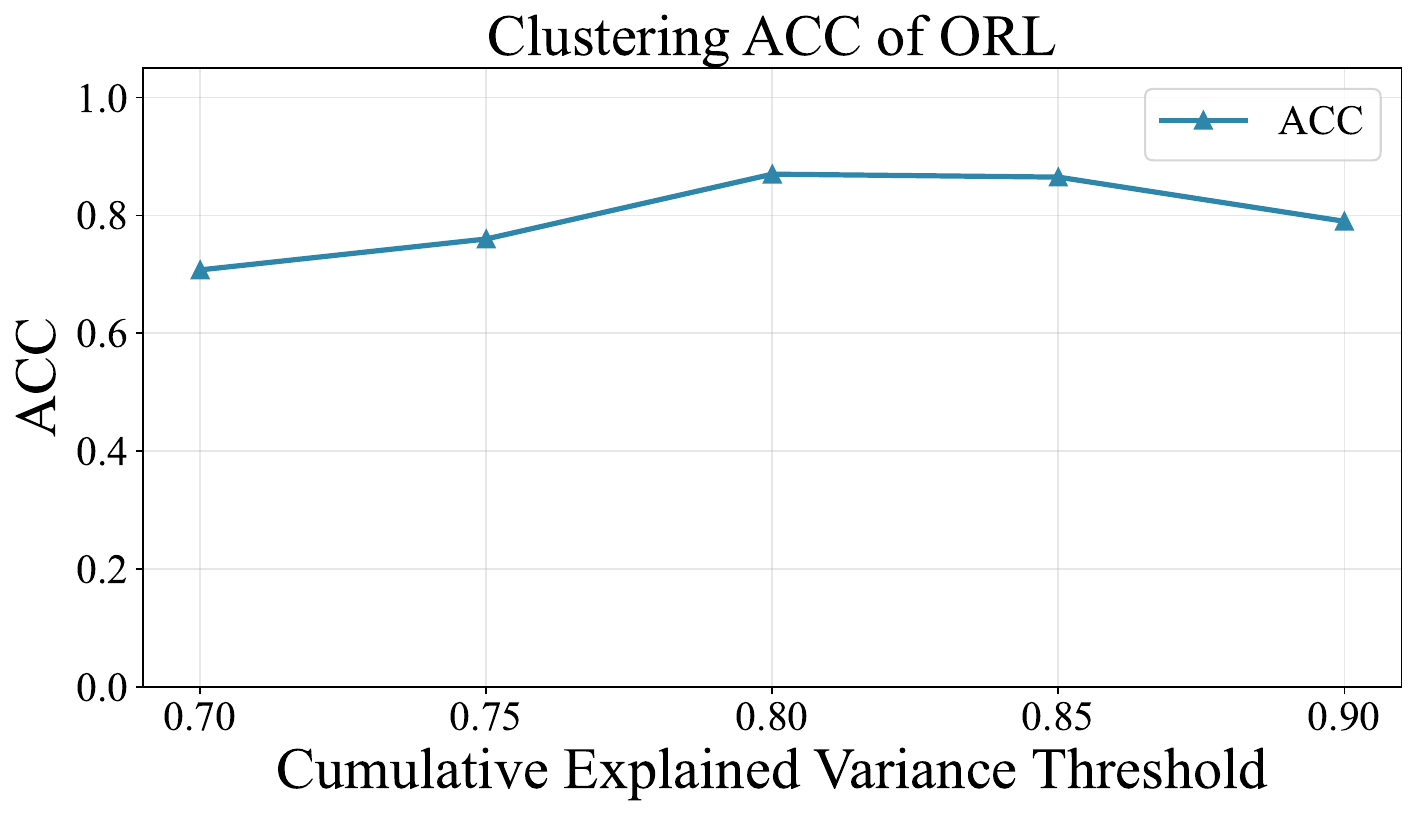}
	\includegraphics[width=0.45\textwidth]{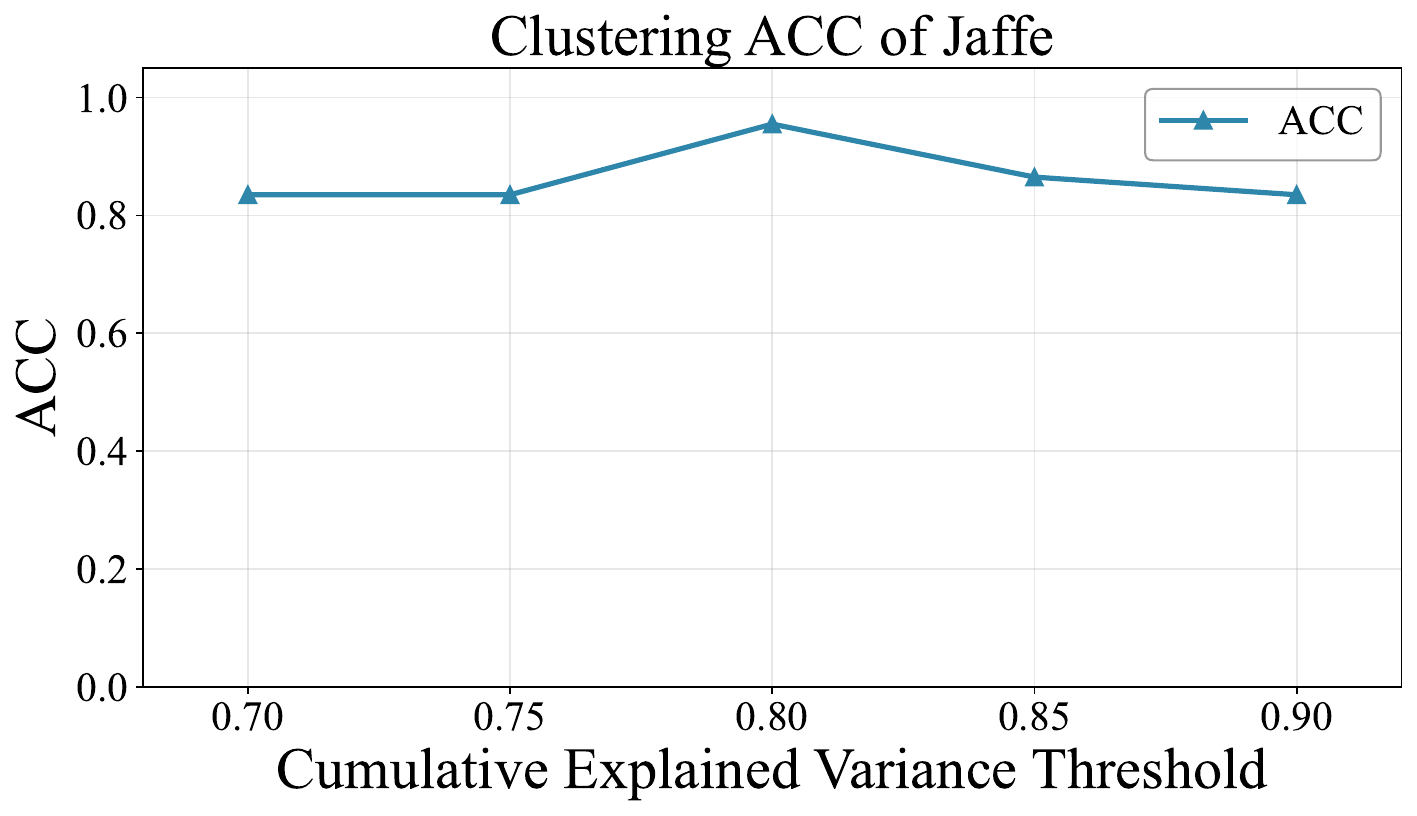}

	\includegraphics[width=0.45\textwidth]{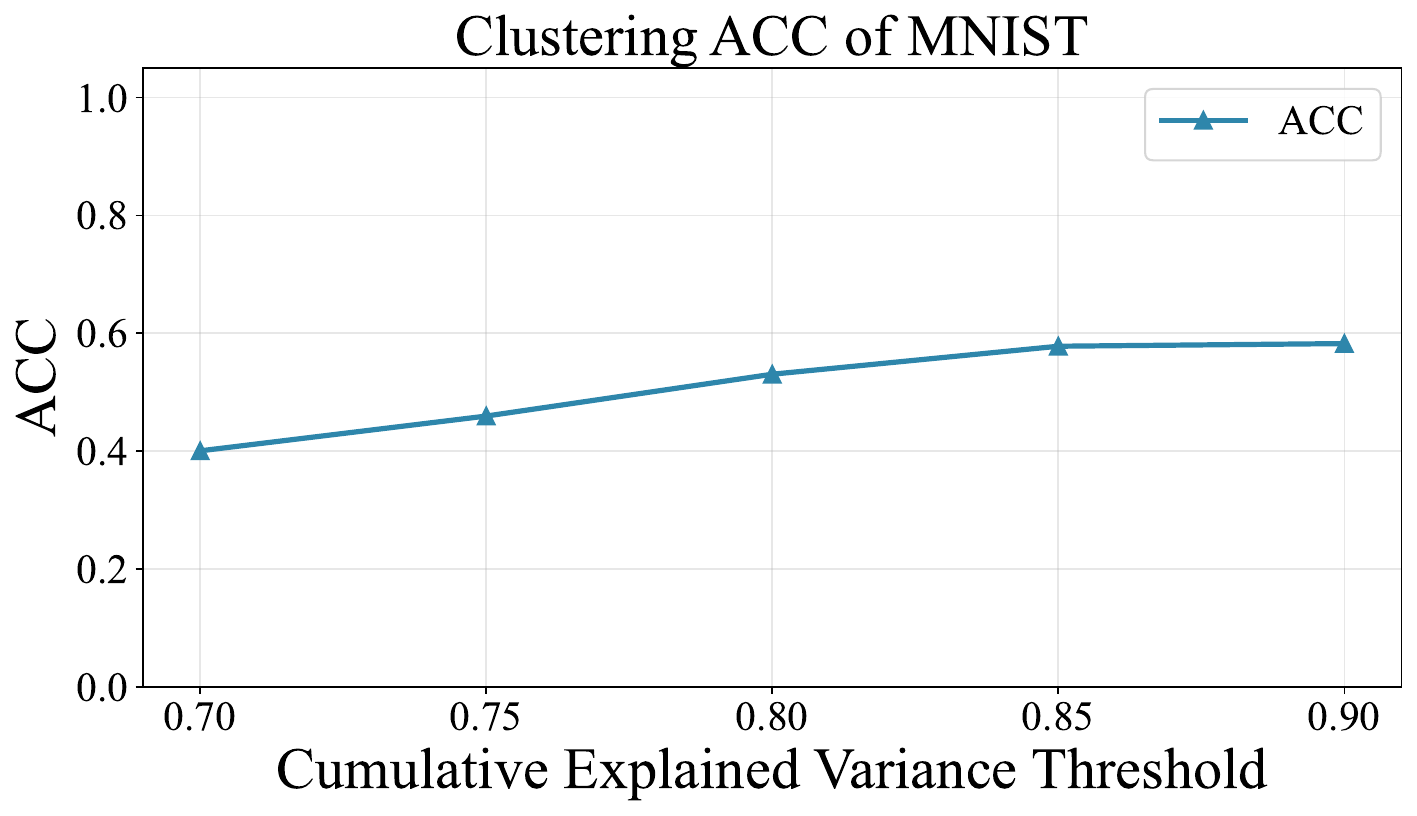}
	\includegraphics[width=0.45\textwidth]{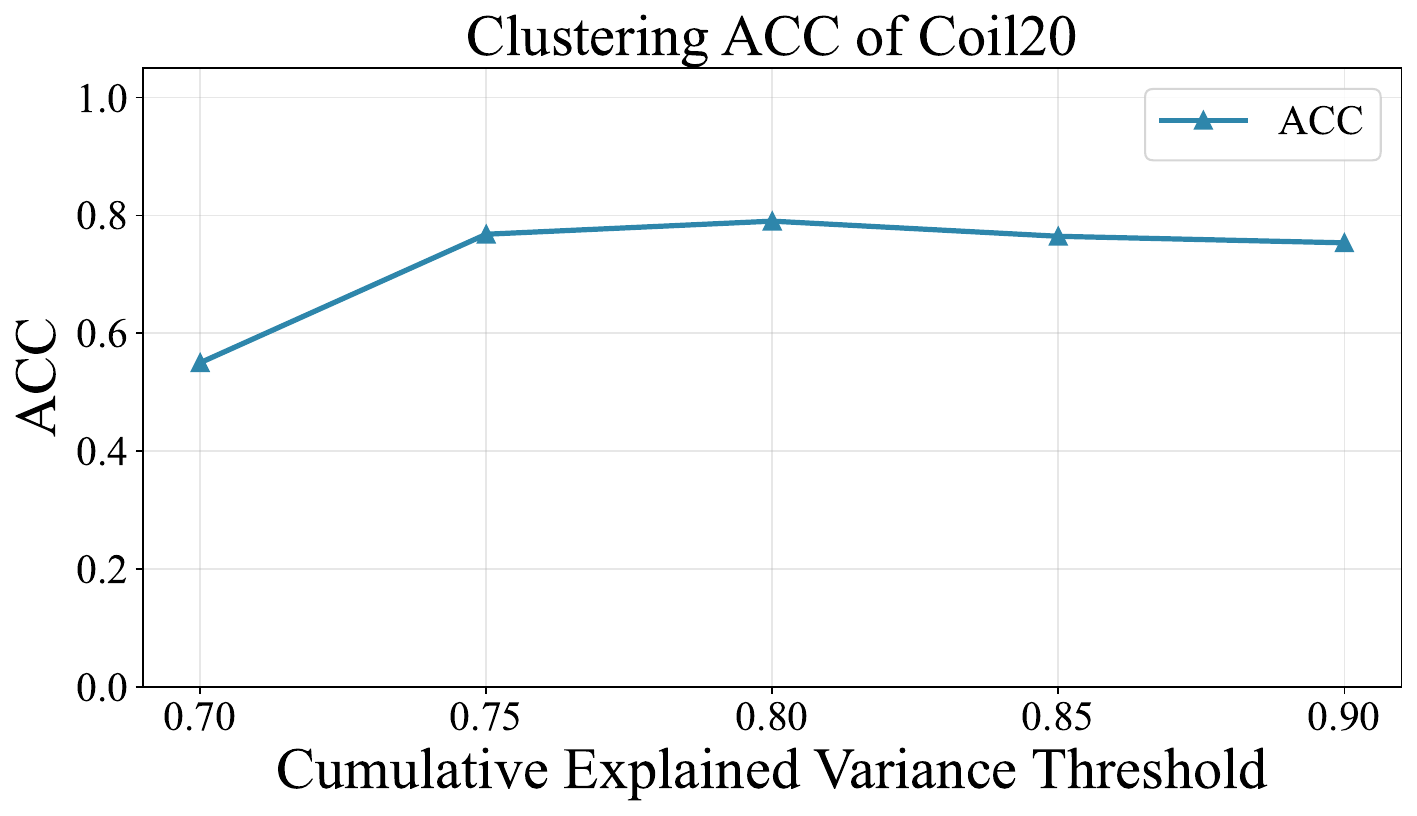}

	\includegraphics[width=0.45\textwidth]{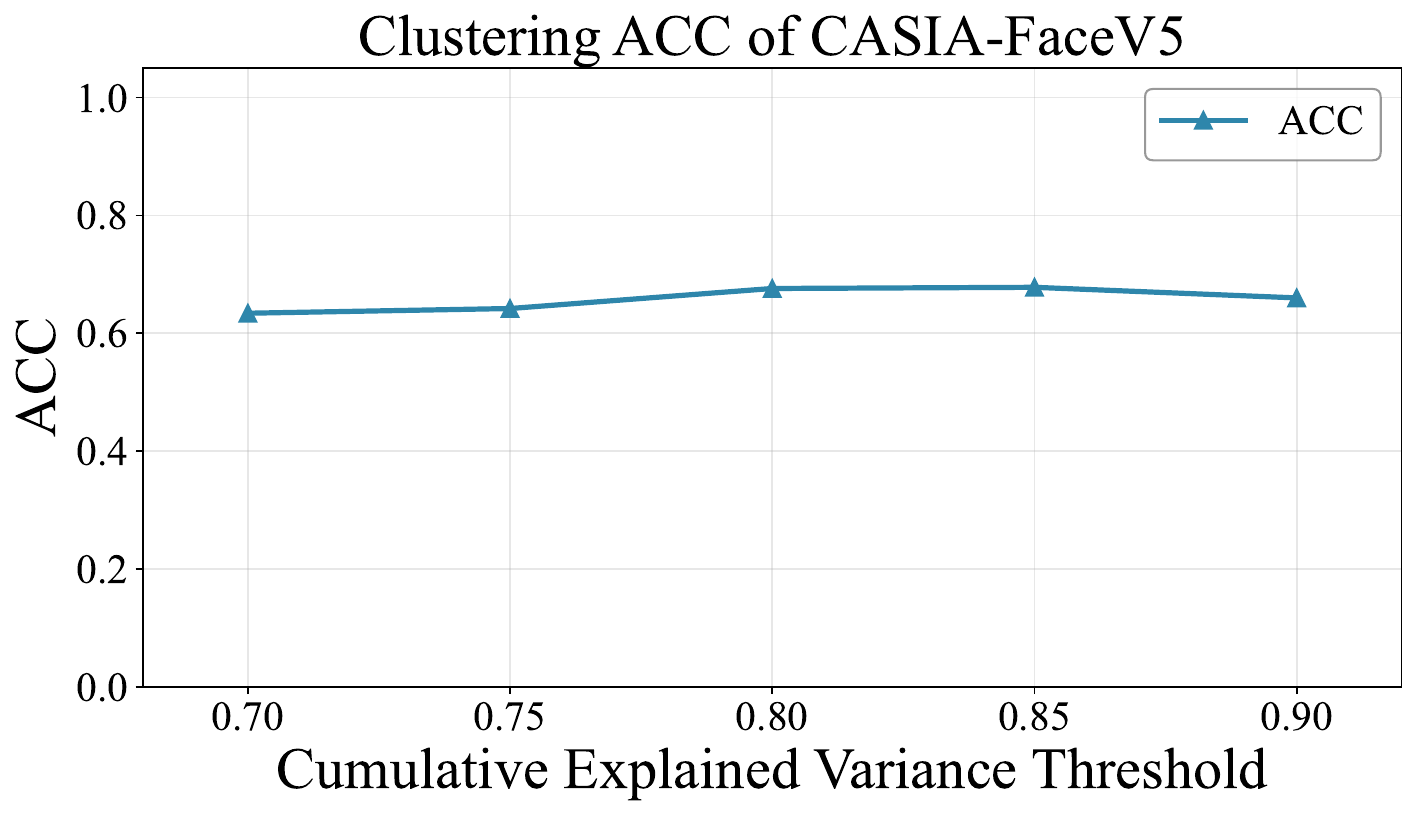}
	\includegraphics[width=0.45\textwidth]{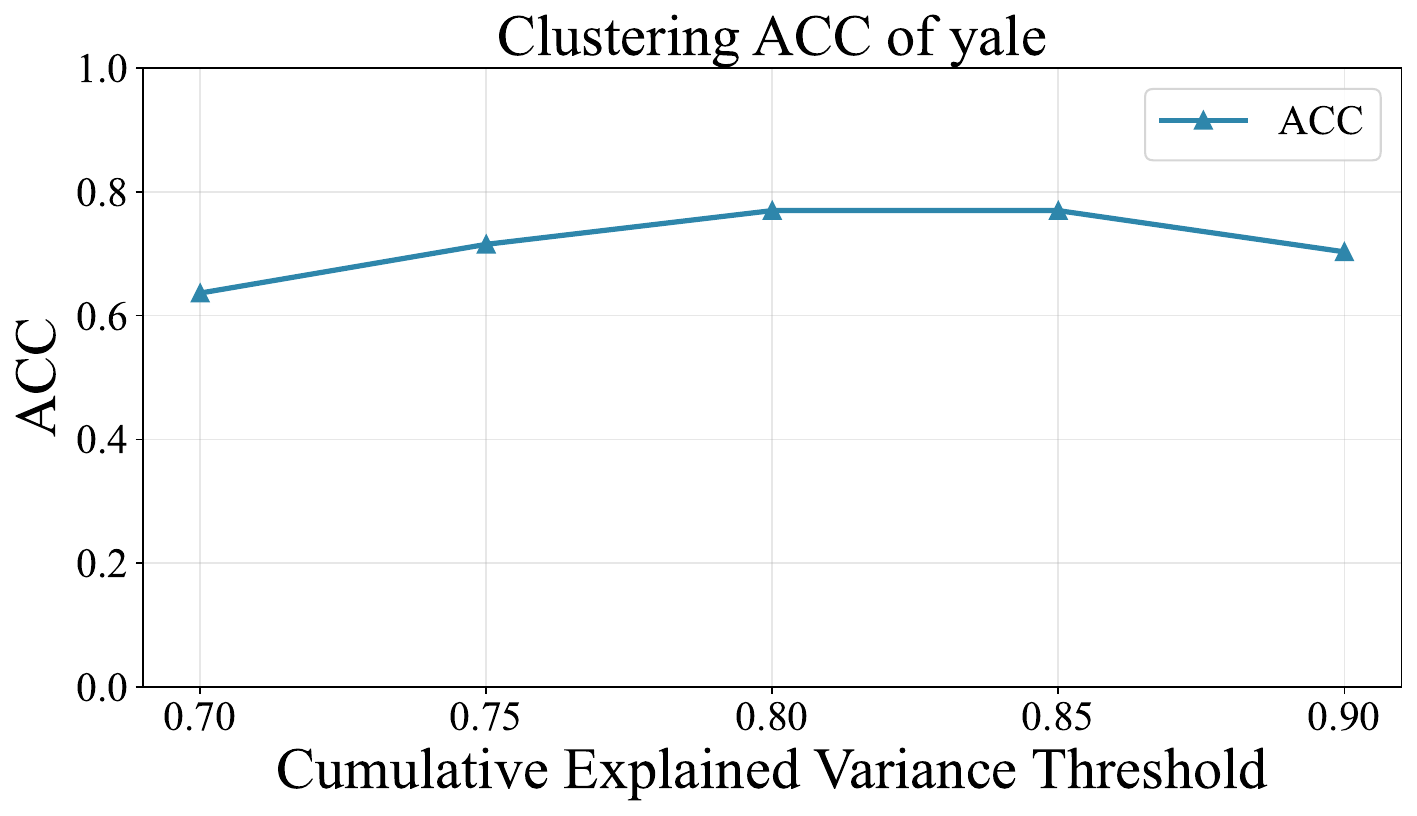}

	\includegraphics[width=0.45\textwidth]{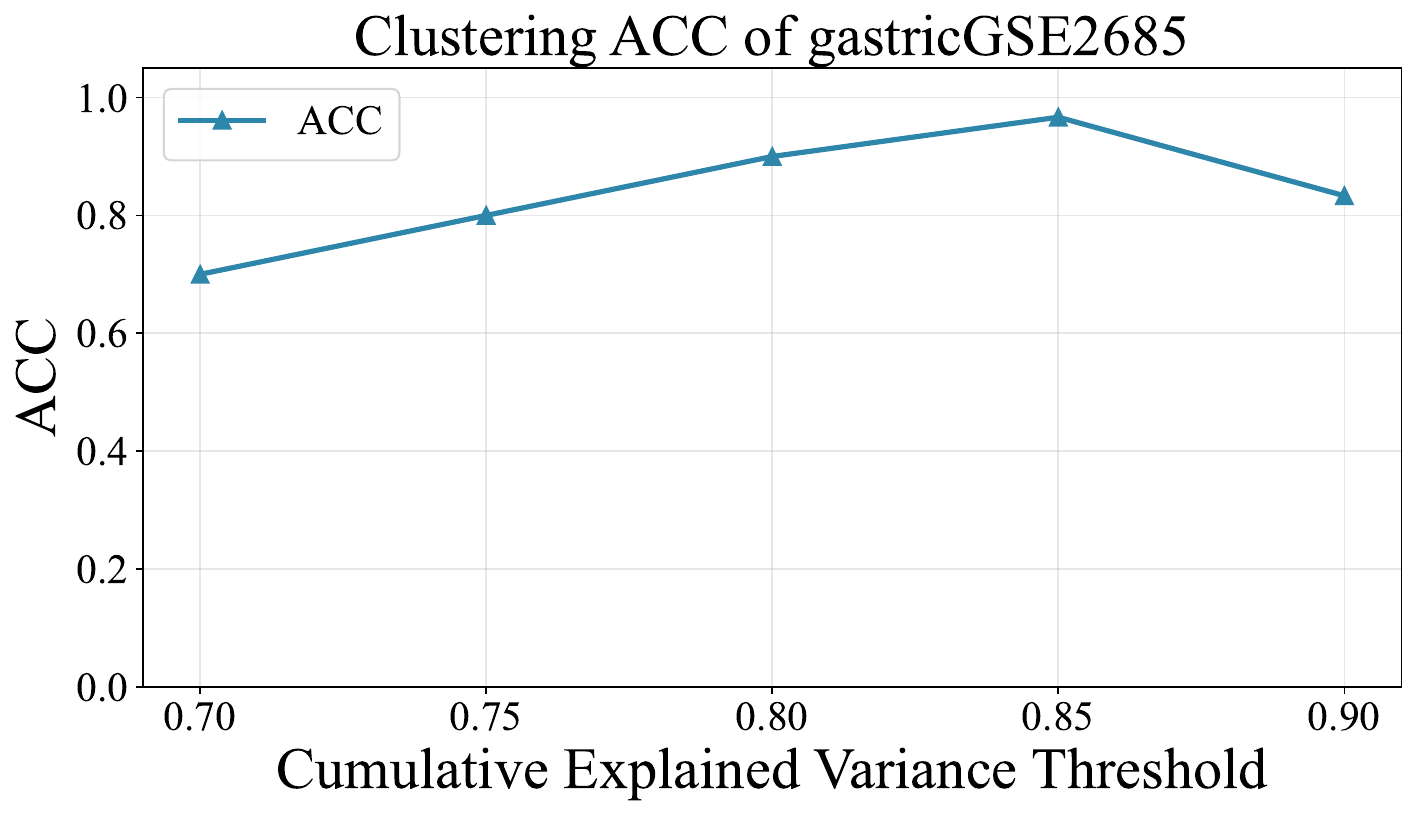}
	\includegraphics[width=0.45\textwidth]{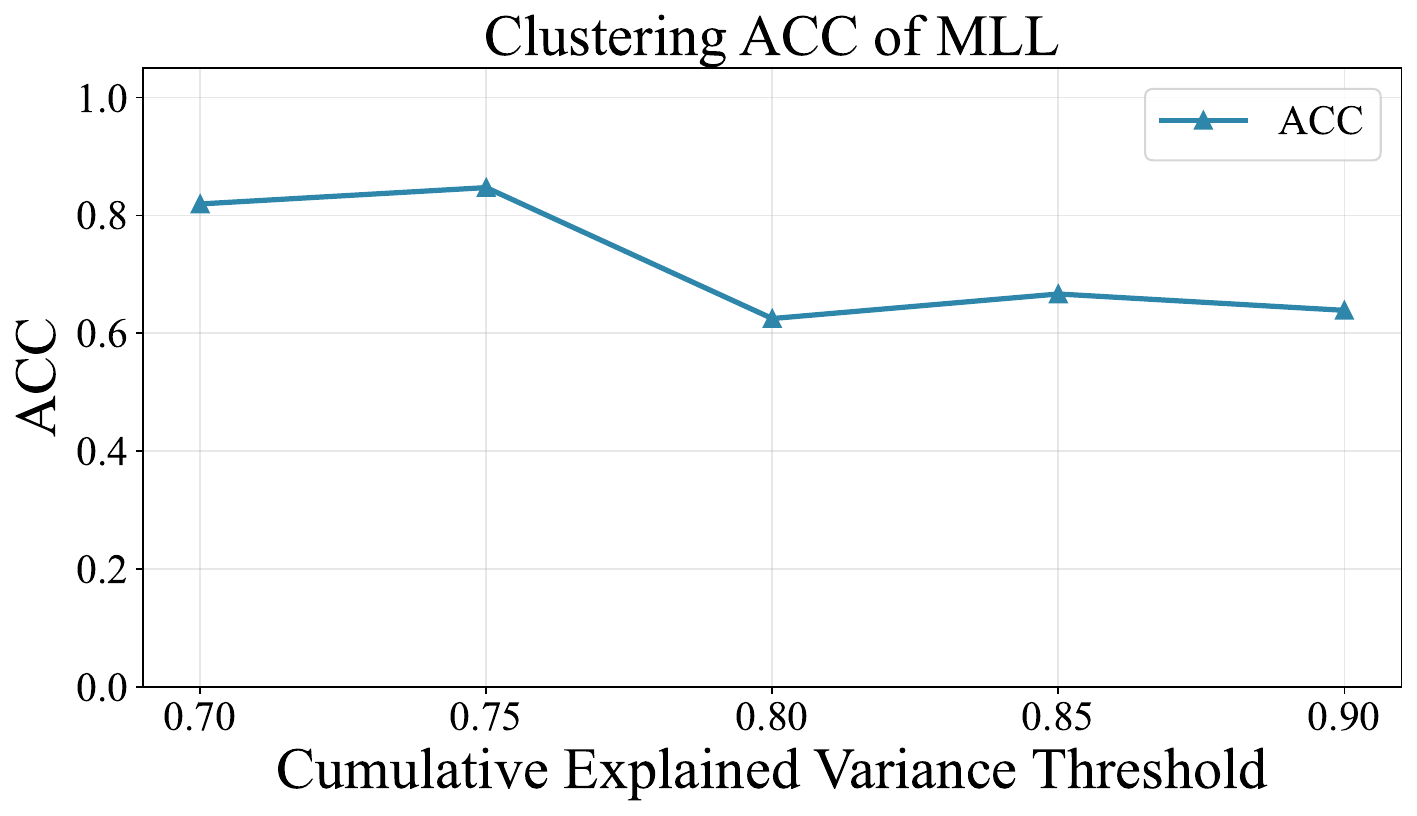}

	\hspace{0.0\linewidth}
	\includegraphics[width=0.45\textwidth]{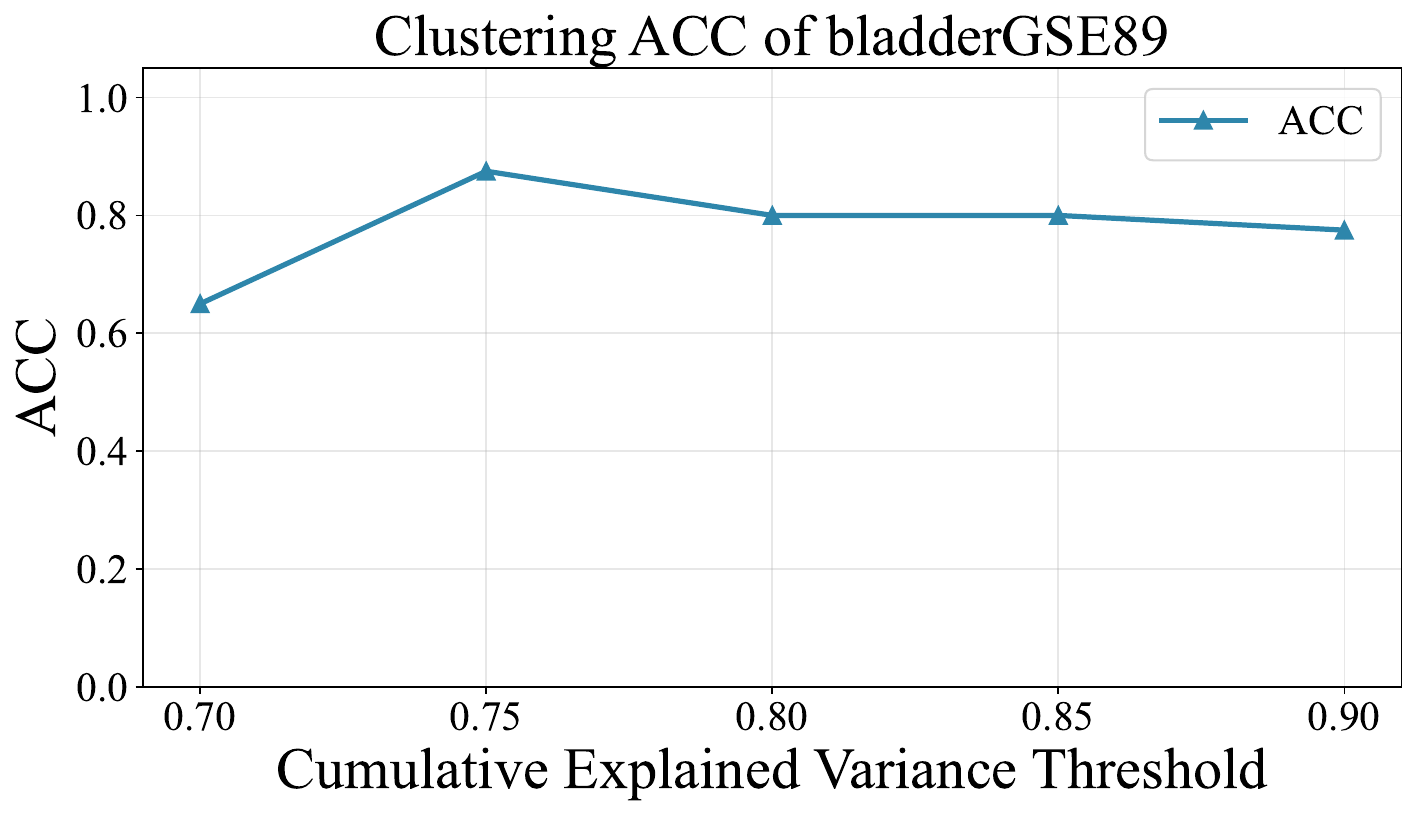}
	\hspace{0.0\linewidth}
	\caption{ACC Comparison Across Datasets at Variance Thresholds}
	\label{fig:clustering_acc_large}
\end{figure}

Significant variations in the ACC performance of different datasets under cumulative variance explanation rate criteria are seen in Figure \ref{fig:clustering_acc_large}.  As the threshold increases from 0.70 to 0.80, the ORL dataset shows a growing trend in ACC, reaching optimal performance at threshold 0.80 (equivalent to subspace dimension 30).  Retaining 80\% of the variance is adequate for successful feature extraction in this dataset, whereas higher proportions introduce redundant information, as further raising the threshold deteriorates performance.  At threshold 0.80 (dimension 28), where discriminative information is fully retrieved, the Jaffe dataset performs at its best.  The MNIST dataset shows a constant increase in ACC with increasing thresholds, culminating at 0.85 (dimension 10), due to the intricacy of handwritten digit features, necessitating a greater variance retention rate to capture feature details. Peak ACC is reached by the Coil20 dataset at a threshold of 0.80 (dimension 10).  As the threshold rises from 0.70 to 0.80, category-distinguishing variance is successfully maintained.  Performance stabilizes or slightly declines beyond 0.80 as the impact of additional variation decreases.  At threshold 0.80 (dimensions 18 and 46, respectively), which already includes enough facial identification features, both the Yale and CASIA-FaceV5 datasets perform at their best.  At threshold 0.85 (dimension 21), the gastricGSE2685 dataset performs at its best.  At a threshold of 0.75 (dimensions 33 and 26, respectively), the MLL and bladderGSE89 datasets perform optimally, suggesting that 75\% cumulative variance explanation already provides adequate categorical information.  Clustering judgments may be hampered by irrelevant variance introduced by overly high criteria.

In conclusion, optimal clustering efficiency is attained at different subspace dimensions (defined by cumulative variance explanation) due to variations in feature complexity and category information distribution across datasets.  The majority of facial datasets (ORL, Jaffe, Yale, CASIA-FaceV5) and object datasets (Coil20) cluster at 0.80. The complex image and the biological datasets (MNIST, gastricGSE2685) necessitate raising the threshold to 0.85; some datasets (bladderGSE89, MLL) perform best at a threshold of 0.75.  These findings support the usefulness of using cumulative variance explanation rate to determine subspace dimensions, which serves as a foundation for parameter selection in clustering tasks involving a variety of datasets.
\subsubsection{Algorithm performance comparison across different sample sizes}
Clustering tests were performed on the ORL dataset (40 categories overall) utilizing samples from $n = 2, 4, 7, 15,$ and $30$ categories, respectively, in order to assess the algorithm's robustness across sample sizes.  The mean and standard deviation of ACC and NMI were recorded after 20 repetitions of each setting.  Tables \ref{tab:orl_acc} and \ref{tab:orl_nmi} display the results.  "Mean" and "SD" stand for the mean and standard deviation, respectively, in the tables.

\begin{table}[htbp]
	\small
	\centering
	\caption{ACC Performance on the ORL Dataset across Sample Sizes}
	\label{tab:orl_acc}
	
	\begin{tabular}{llccccc}
		\toprule  
		Number of Subjects & Metric & SSC & LRR & FKM & RKM & OSC \\  
		\midrule  
		
		\multirow{2}{*}{2 subjects} & mean & \textbf{1.000} & 0.870 & 0.940 & 0.960 & \textbf{1.000} \\
		& SD   & \textbf{0.000} & 0.000 & 0.094 & 0.075 & \textbf{0.000} \\
		
		\multirow{2}{*}{4 subjects} & mean & 0.975 & 0.851 & 0.831 & 0.880 & \textbf{1.000} \\
		& SD   & \textbf{0.000} & 0.005 & 0.141 & 0.146 & \textbf{0.000} \\
		
		\multirow{2}{*}{7 subjects} & mean & 0.957 & 0.841 & 0.661 & 0.718 & \textbf{0.998} \\
		& SD   & \textbf{0.000} & 0.005 & 0.105 & 0.107 & 0.006 \\
		
		\multirow{2}{*}{15 subjects} & mean & 0.862 & 0.839 & 0.656 & 0.659 & \textbf{0.878} \\
		& SD   & 0.028 & \textbf{0.014} & 0.072 & 0.044 & 0.027 \\
		
		\multirow{2}{*}{30 subjects} & mean & 0.851 & 0.824 & 0.617 & 0.629 & \textbf{0.865} \\
		& SD   & \textbf{0.003} & 0.023 & 0.043 & 0.048 & 0.024 \\
		\bottomrule  
	\end{tabular}
\end{table}

In terms of clustering accuracy (ACC), Table \ref{tab:orl_acc} demonstrates that when the sample size is modest ($n=2,4,7$), the OSC algorithm performs near or around 1.000, greatly beating comparator algorithms like SSC, LRR, FKM, and RKM.  In particular, the OSC method attains an average ACC of 1.000 for $n=2$, which is greater than LRR (0.870), FKM (0.940), and RKM (0.960) but comparable to SSC (1.000).  The OSC algorithm outperforms SSC (0.975) and LRR (0.851) when $n=4$, maintaining an ACC of 1.000.  Every algorithm shows performance decrease as the sample size grows to $n=15$ and $30$.  However, the OSC method shows better scalability by maintaining the highest mean values (0.878 and 0.865, respectively).  The OSC method typically shows lower standard deviation values, indicating improved stability.
\begin{table}[htbp]
	\small
	\centering
	\caption{NMI Performance on the ORL Dataset across Sample Sizes}
	\label{tab:orl_nmi}
	
	\begin{tabular}{llccccc}
		\toprule  
		Number of Subjects & Metric & SSC & LRR & FKM & RKM & OSC \\
		\midrule  
		
		\multirow{2}{*}{2 subjects} & mean & \textbf{1.000} & 0.881 & 0.839 & 0.871 & \textbf{1.000} \\
		& SD   & \textbf{0.000} & 0.000 & 0.259 & 0.221 & \textbf{0.000} \\

		\multirow{2}{*}{4 subjects} & mean & 0.940 & 0.860 & 0.817 & 0.862 & \textbf{1.000} \\
		& SD   & 0.039 & \textbf{0.000} & 0.110 & 0.141 & \textbf{0.000} \\

		\multirow{2}{*}{7 subjects} & mean & 0.927 & 0.822 & 0.817 & 0.818 & \textbf{0.998} \\
		& SD   & 0.037 & 0.009 & 0.072 & 0.064 & 0.008 \\

		\multirow{2}{*}{15 subjects} & mean & \textbf{0.963} & 0.898 & 0.810 & 0.801 & 0.933 \\
		& SD   & 0.018 & \textbf{0.012} & 0.044 & 0.034 & 0.014 \\

		\multirow{2}{*}{30 subjects} & mean & 0.911 & 0.887 & 0.769 & 0.795 & \textbf{0.928} \\
		& SD   & 0.009 & \textbf{0.008} & 0.023 & 0.019 & 0.009 \\
		\bottomrule  
	\end{tabular}
\end{table}

In terms of the normalized mutual information (NMI) metric, Table \ref{tab:orl_nmi} demonstrates that the OSC algorithm outperforms other comparison algorithms, achieving an NMI mean of 1.000 for smaller sample sizes ($n=2,4$).  In particular, the SSC method produces an NMI mean of 1.000 when $n=2$, but LRR produces 0.881 and FKM 0.839.  The suggested approach achieves an NMI mean of 1.000 when $n=4$, exceeding SSC (0.940).  All algorithms' NMI values typically drop with increasing sample size.  Nonetheless, the suggested approach retains the greatest values for $n=7$ and $n=30$ (0.998 and 0.928, respectively), demonstrating its capacity to successfully detect cluster structures across a range of sample sizes.  Significant drops in the NMI mean and higher standard deviations were observed in the FKM and RKM algorithms, suggesting comparatively lower stability. In contrast, the proposed method demonstrated smaller standard deviations and more stable clustering performance.

Overall, all algorithms' clustering performance decreases with increasing sample size.  Nonetheless, the approach presented in this study exhibits exceptional stability and scalability, maintaining ideal or nearly optimal ACC and NMI values in the majority of cases with minor result changes.

\subsubsection{Analysis of OSC algorithm convergence speed}
To verify the convergence of the proposed OSC algorithm, convergence curves depicting the objective function value versus iteration count were plotted on the ORL, Coil20, Yale, and CASIA-FaceV5 datasets. The ORL dataset converges quickly, as seen in Figure \ref{fig:kmeans_convergence_curve}, where the objective function value significantly decreases and stabilizes after four rounds. The Yale dataset similarly experiences a descent from high to low values, with a notable drop early on followed by stabilization, converging to a steady state after 14 iterations; the Coil20 dataset shows a gradual decline in objective function values, with a rapid decrease initially slowing down later, ultimately stabilizing after 22 iterations; The objective function value of the CASIA-FaceV5 dataset is high in the early iterations before rapidly declining. After just three rounds, it converges to a stable number that doesn't significantly fluctuate in subsequent iterations.

The OSC algorithm presented in this study converges quickly on a variety of high-dimensional datasets, according to a thorough investigation, and the objective function value stays constant after convergence. This provides effective computational assurance for real-world clustering applications and supports the algorithm's strong convergence features.

\begin{figure}[htbp]
	\centering
	
	\includegraphics[width=0.45\textwidth]{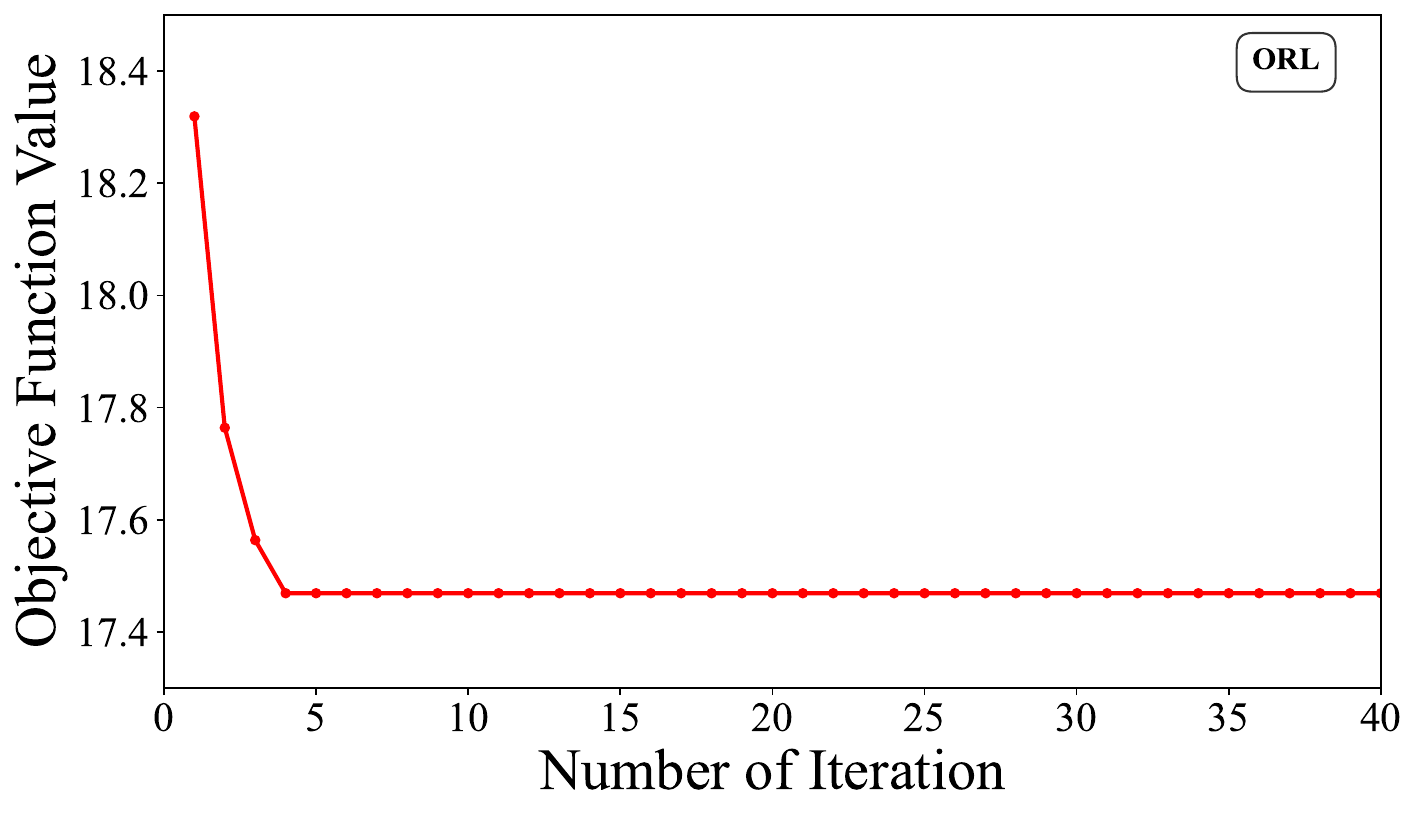}
	\hspace{0.0\linewidth}
	\includegraphics[width=0.45\textwidth]{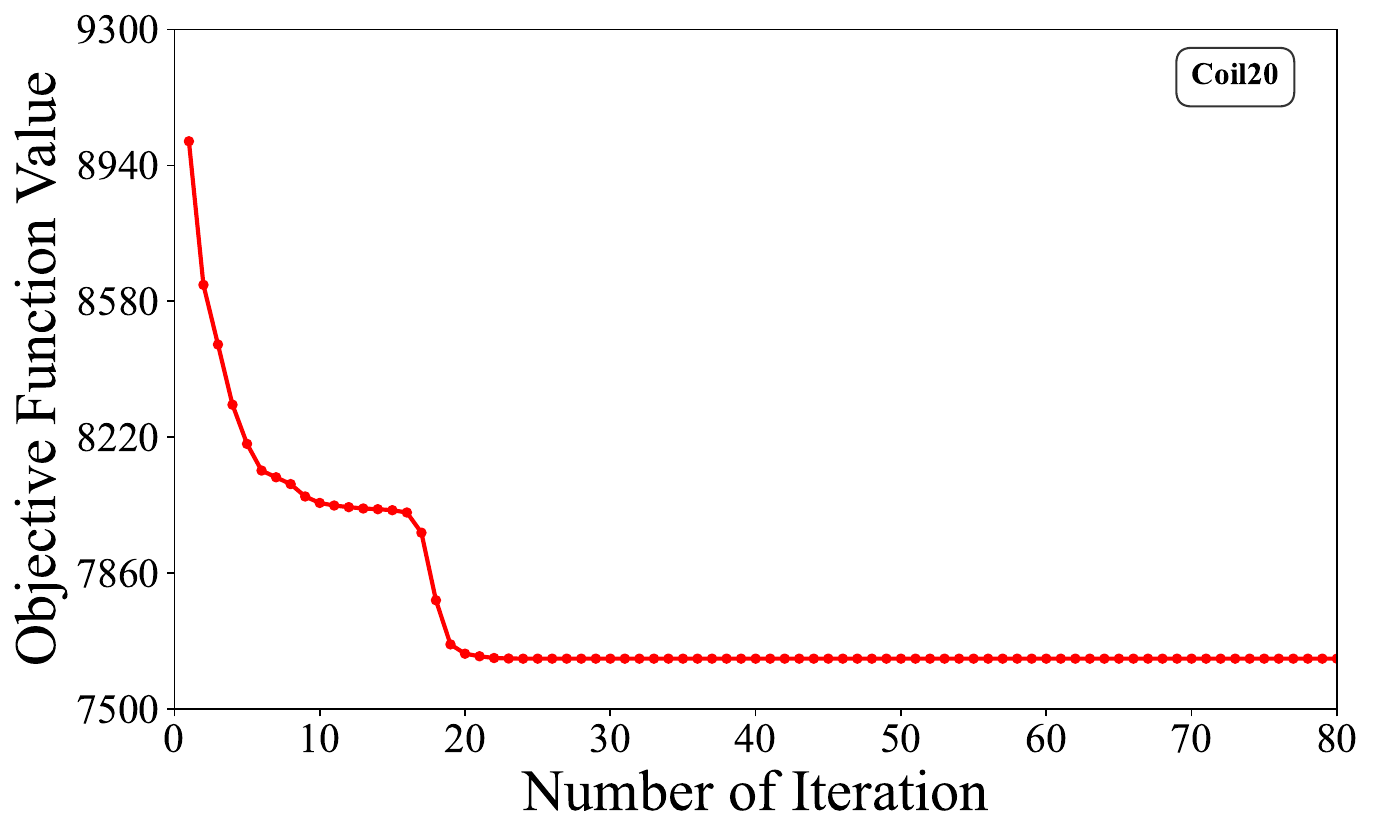}

	\includegraphics[width=0.45\textwidth]{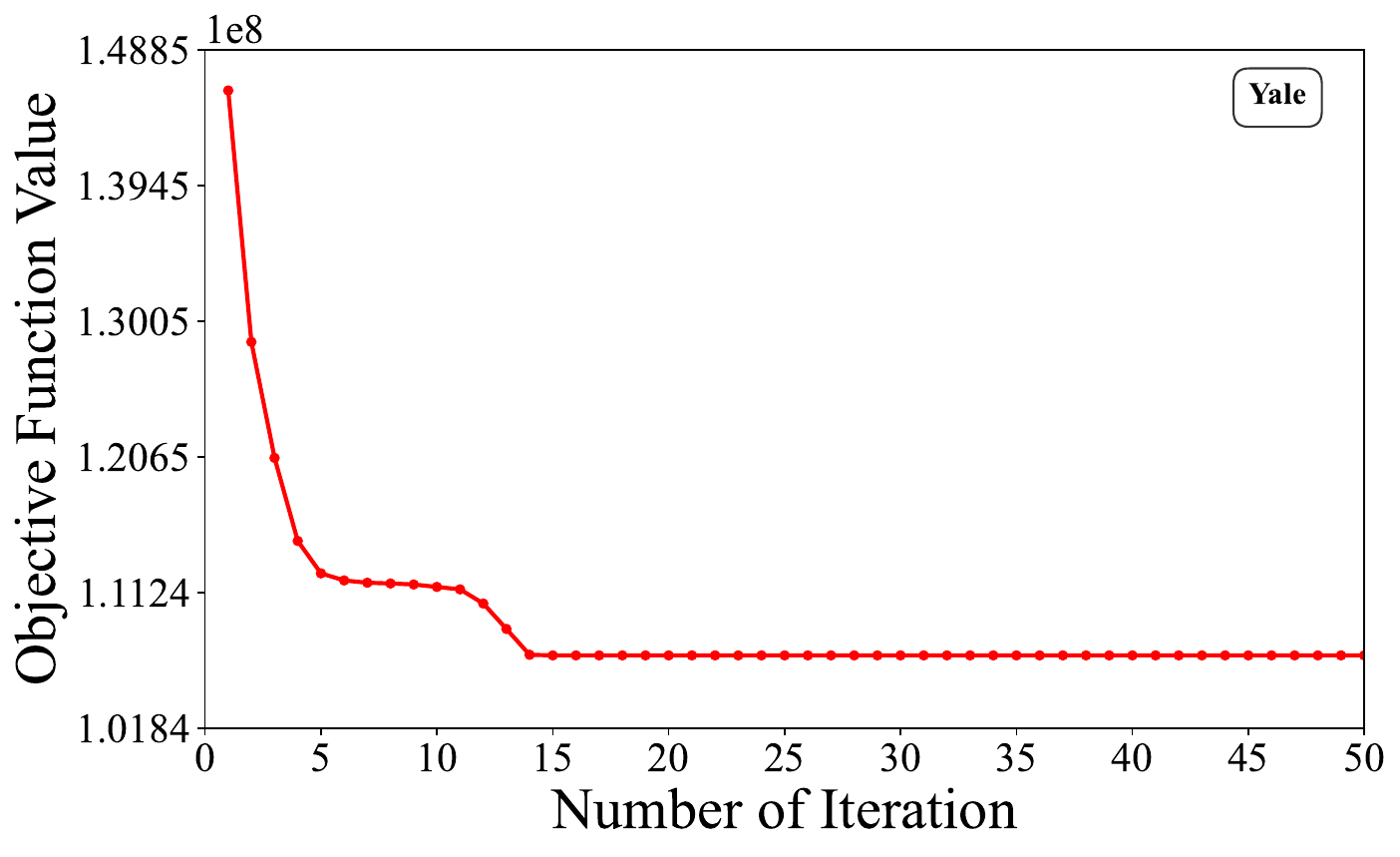}
	\hspace{0.0\linewidth}
	\includegraphics[width=0.45\textwidth]{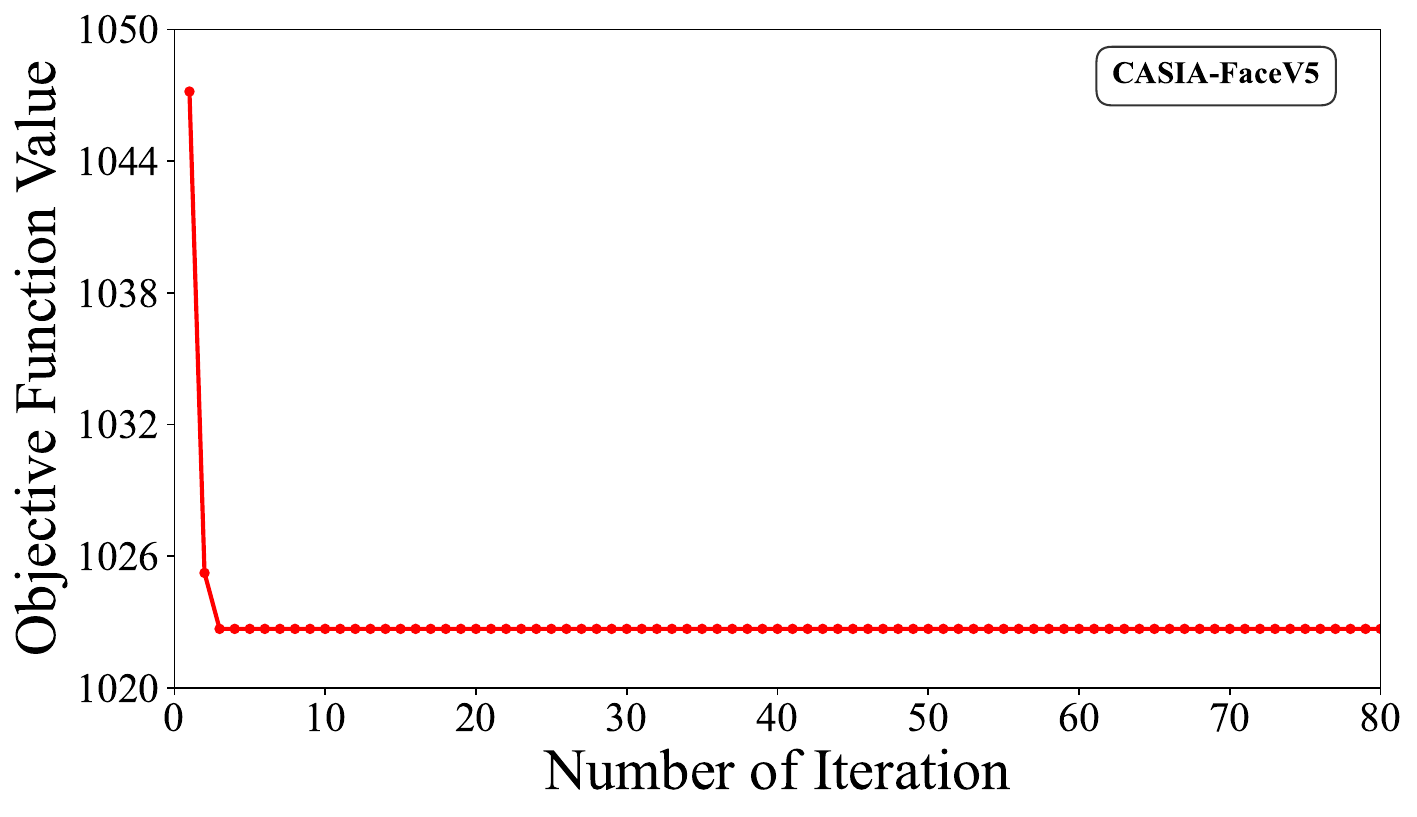}
	
	\caption{Convergence Speed of the OSC Algorithm across Datasets}
	\label{fig:kmeans_convergence_curve}
\end{figure}

\subsubsection{Comparison examination of algorithm computational efficiency}
To evaluate the temporal efficiency of different algorithms, the average computation time of the LRR, SSC, FKM, RKM, and OSC algorithms was compared on nine datasets, namely bladderGSE89, CASIA-FaceV5, Coil20, gastricGSE2685, Jaffe, MLL, MNIST, ORL, and Yale. The comparison results are presented in Figure \ref{fig:runtime_comparison_part}, and the sample sizes corresponding to each dataset are given in Equation (\ref{eq:sample_sizes}).The sample sizes for the nine datasets are configured according to \(s_1, s_2, s_3, s_4, s_5, s_6, s_7, s_8, s_9\) in Equation (\ref{eq:sample_sizes}).
\begin{equation}
	\label{eq:sample_sizes}
	\begin{aligned}
		s_1 &= \{2, 4, 7, 15, 30\}, \quad &s_2 &= \{2, 3, 5, 8, 10\}, \quad &s_3 &= \{2, 3, 4, 5, 6\}, \\
		s_4 &= \{2, 3, 5, 8, 10\}, \quad &s_5 &= \{2, 4, 7, 10, 15\}, \quad &s_6 &= \{20, 40, 60, 80, 100\}, \\
		s_7 &= \{1, 2, 3\}, \quad &s_8 &= \{1, 2\}, \quad &s_9 &= \{1, 2, 3\}
	\end{aligned}
\end{equation}

\begin{figure}[htbp!]
	\centering
	
	\includegraphics[width=0.45\textwidth]{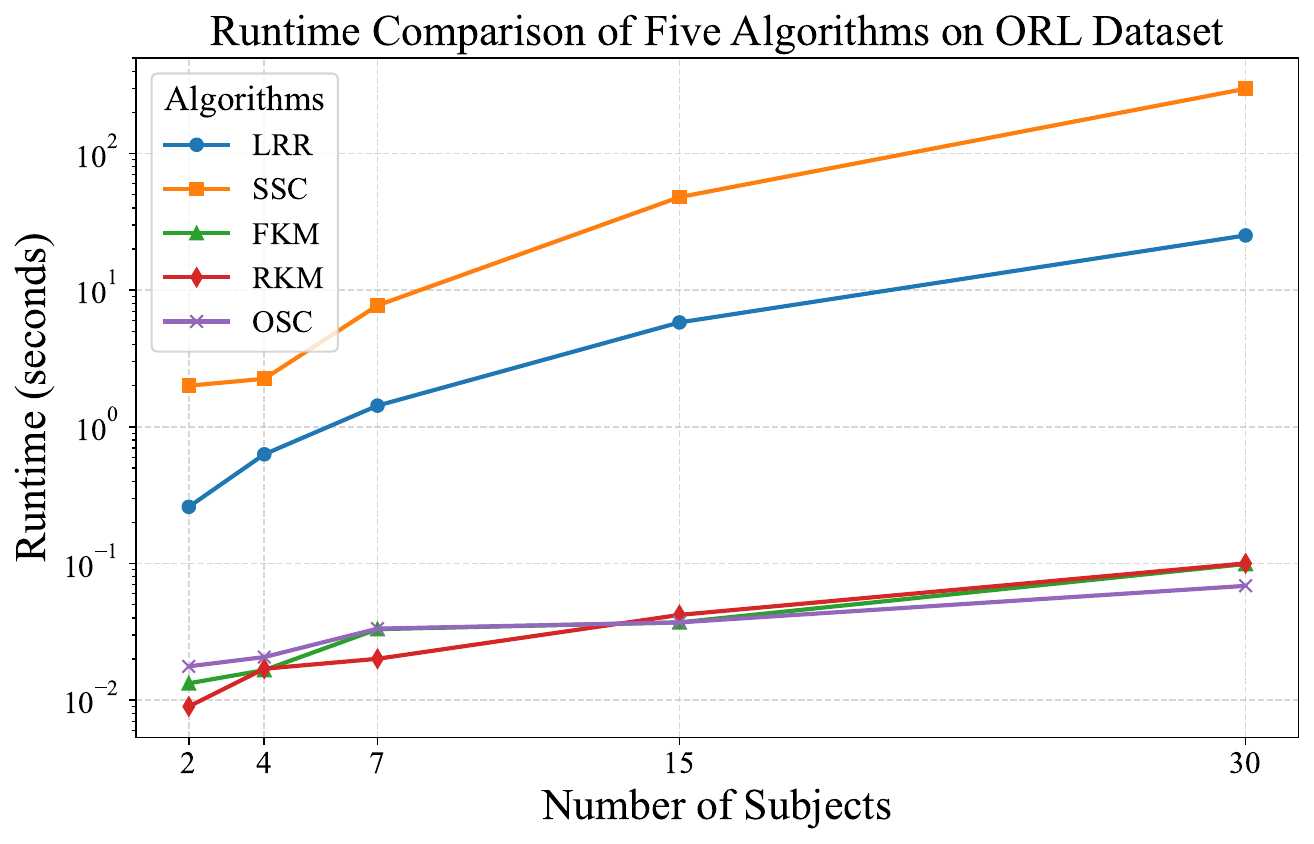}
	\includegraphics[width=0.45\textwidth]{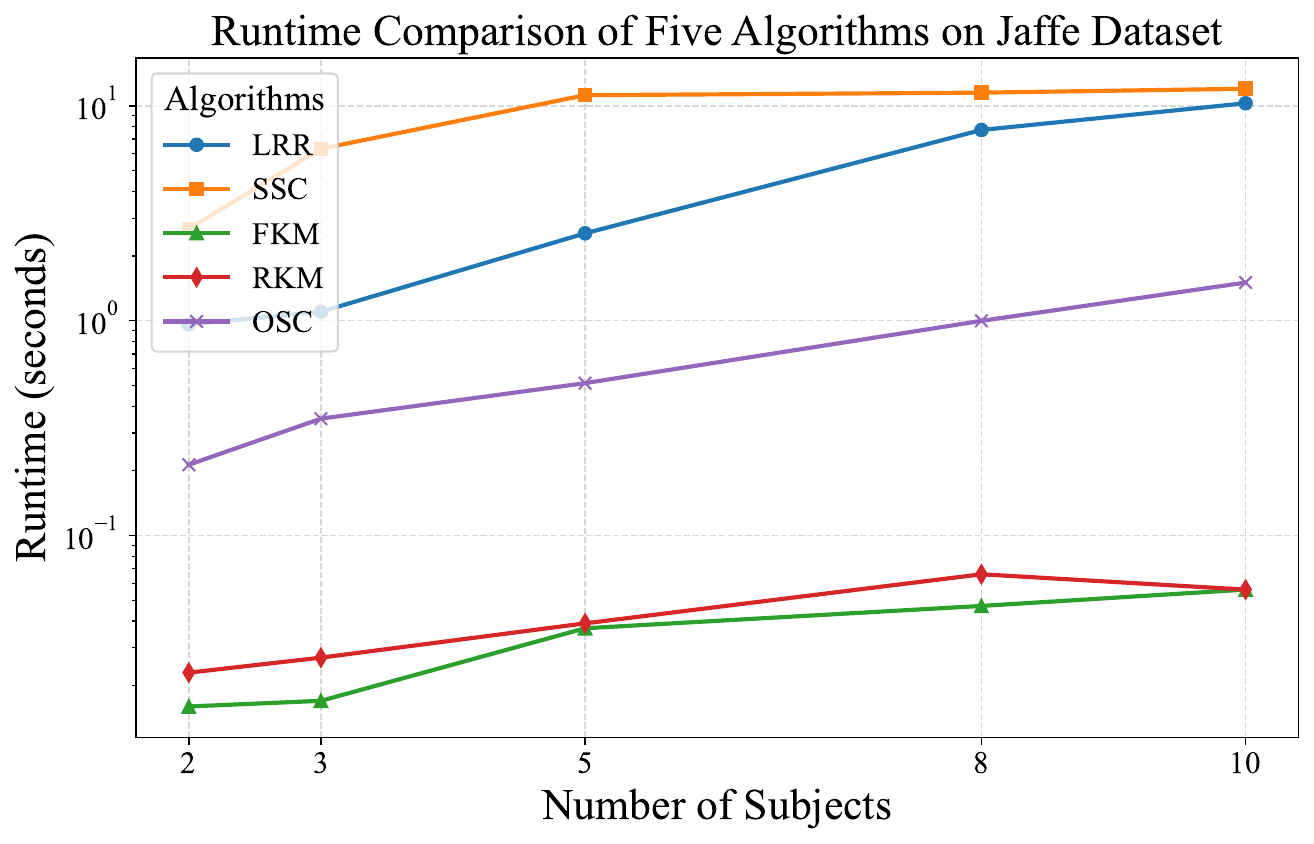}

	\includegraphics[width=0.45\textwidth]{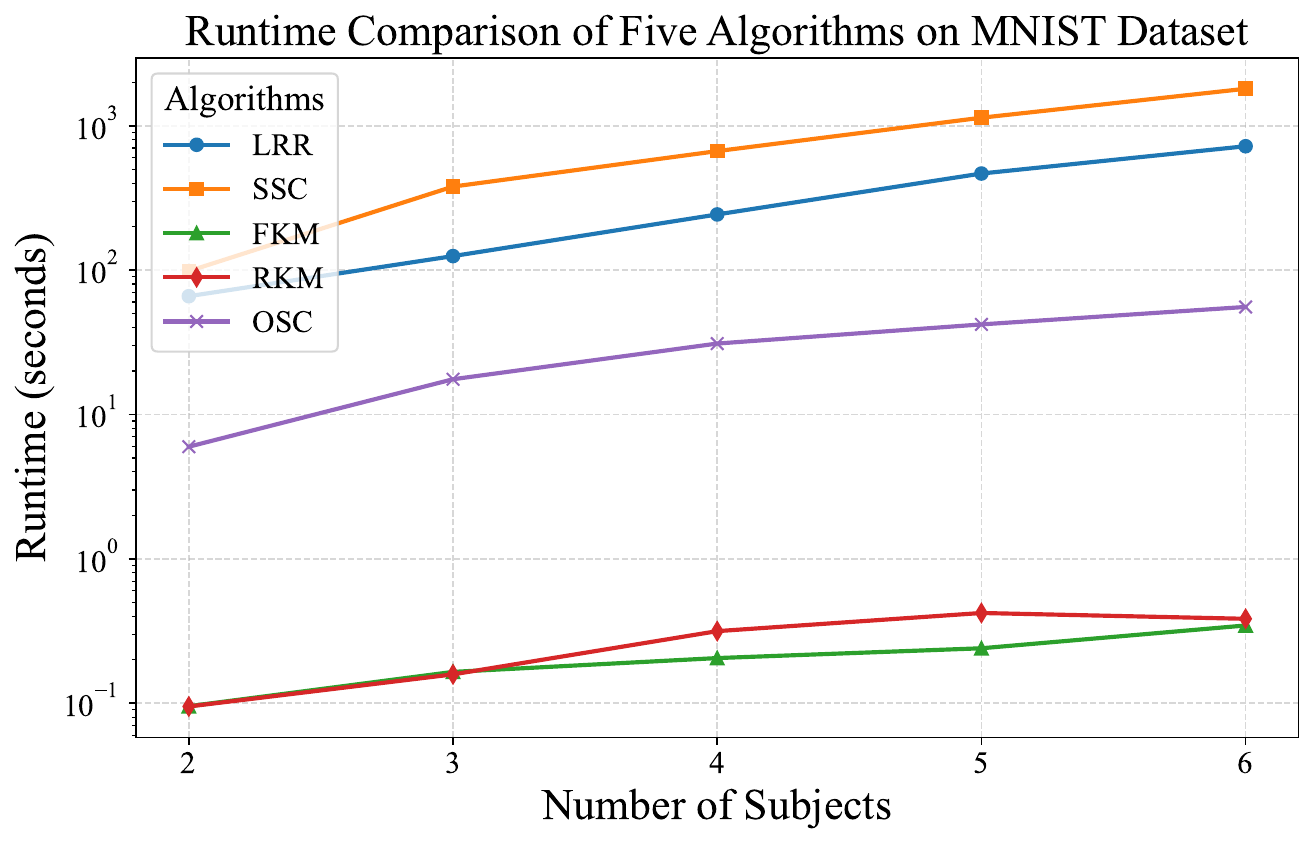}
	\includegraphics[width=0.45\textwidth]{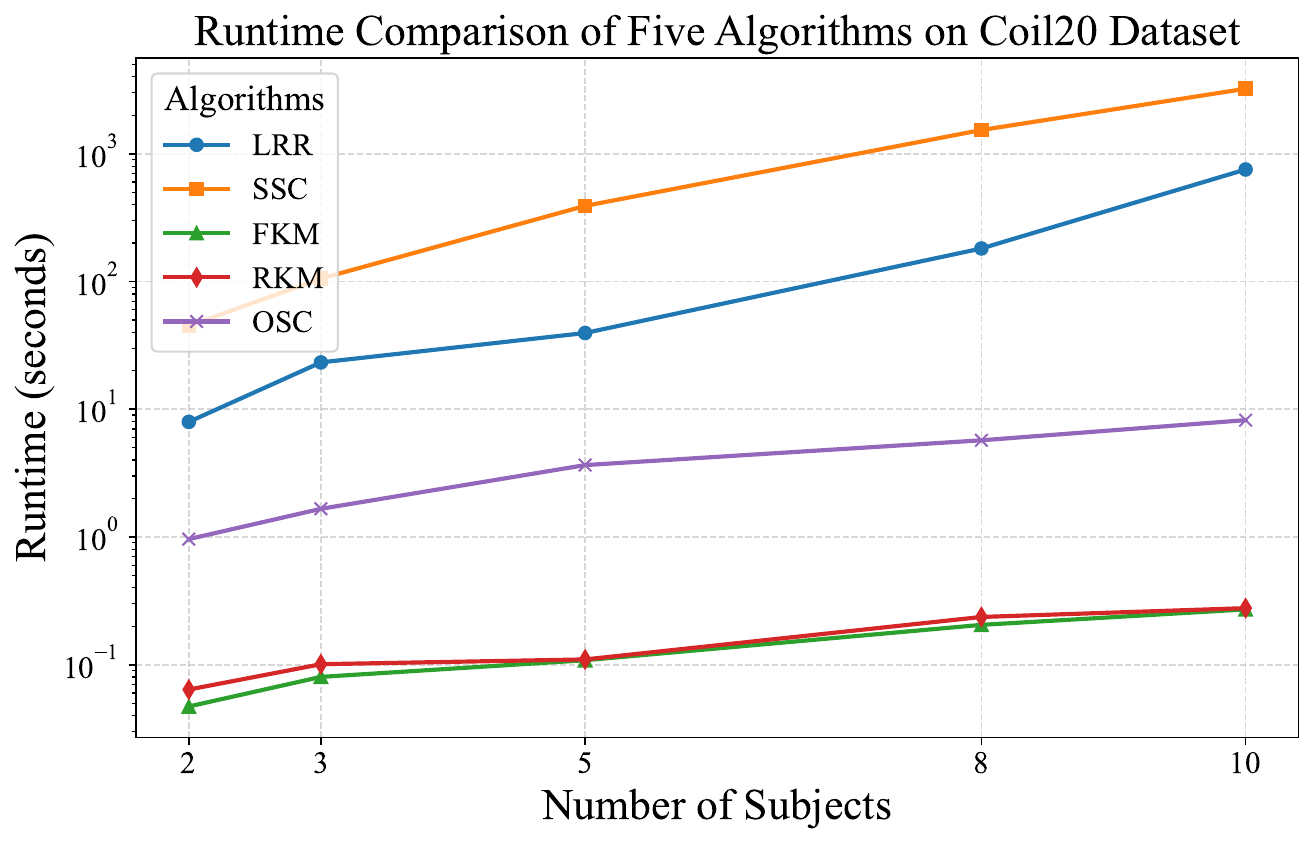}

	\includegraphics[width=0.45\textwidth]{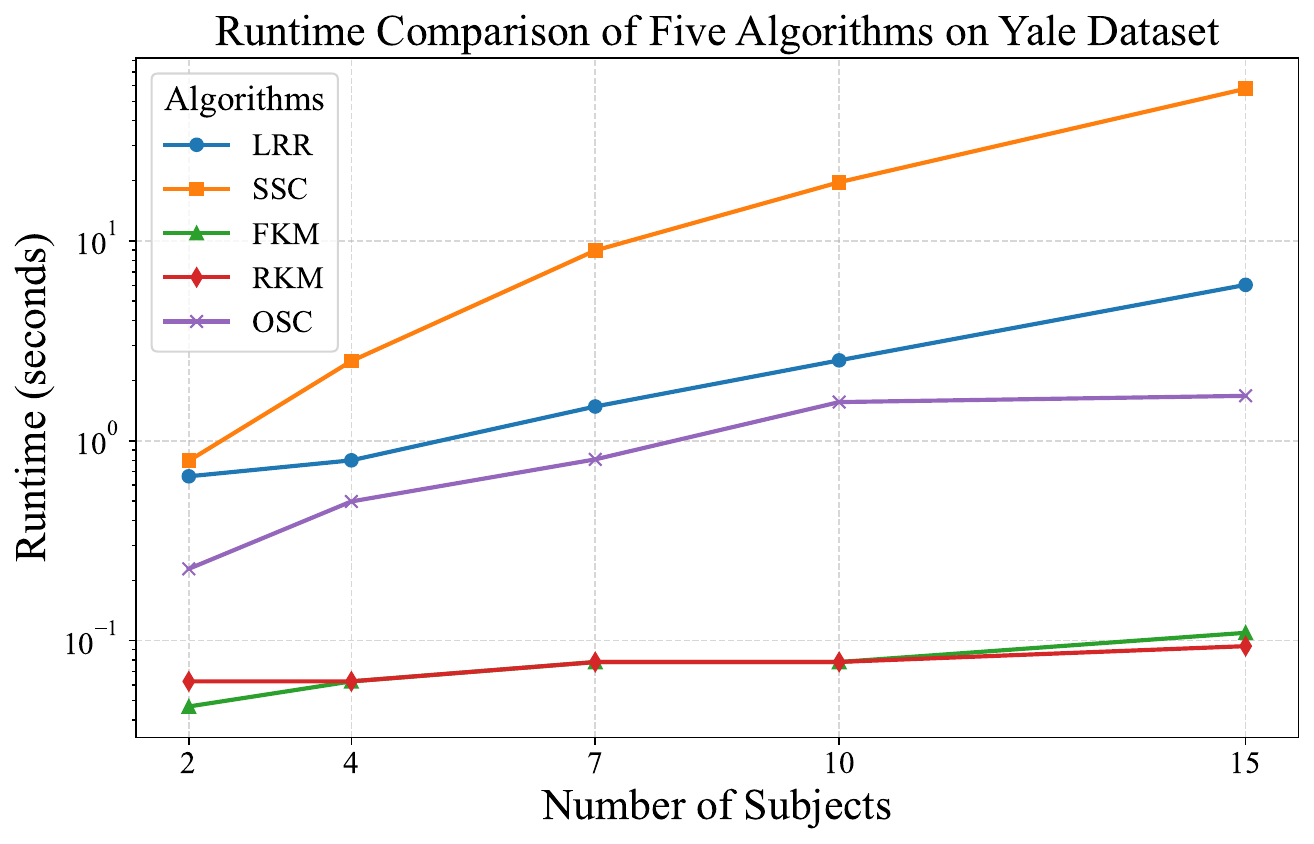}
	\includegraphics[width=0.45\textwidth]{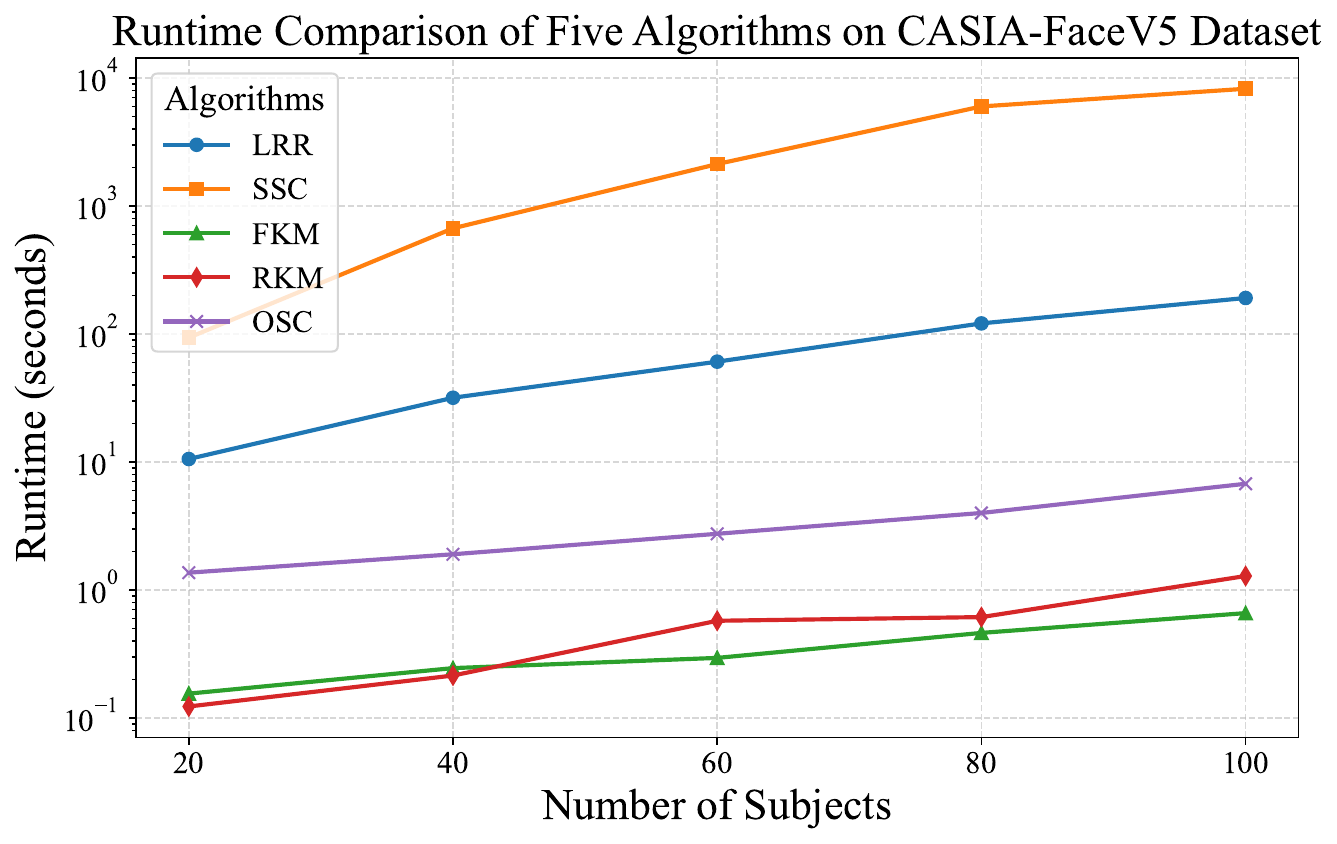}

	\includegraphics[width=0.45\textwidth]{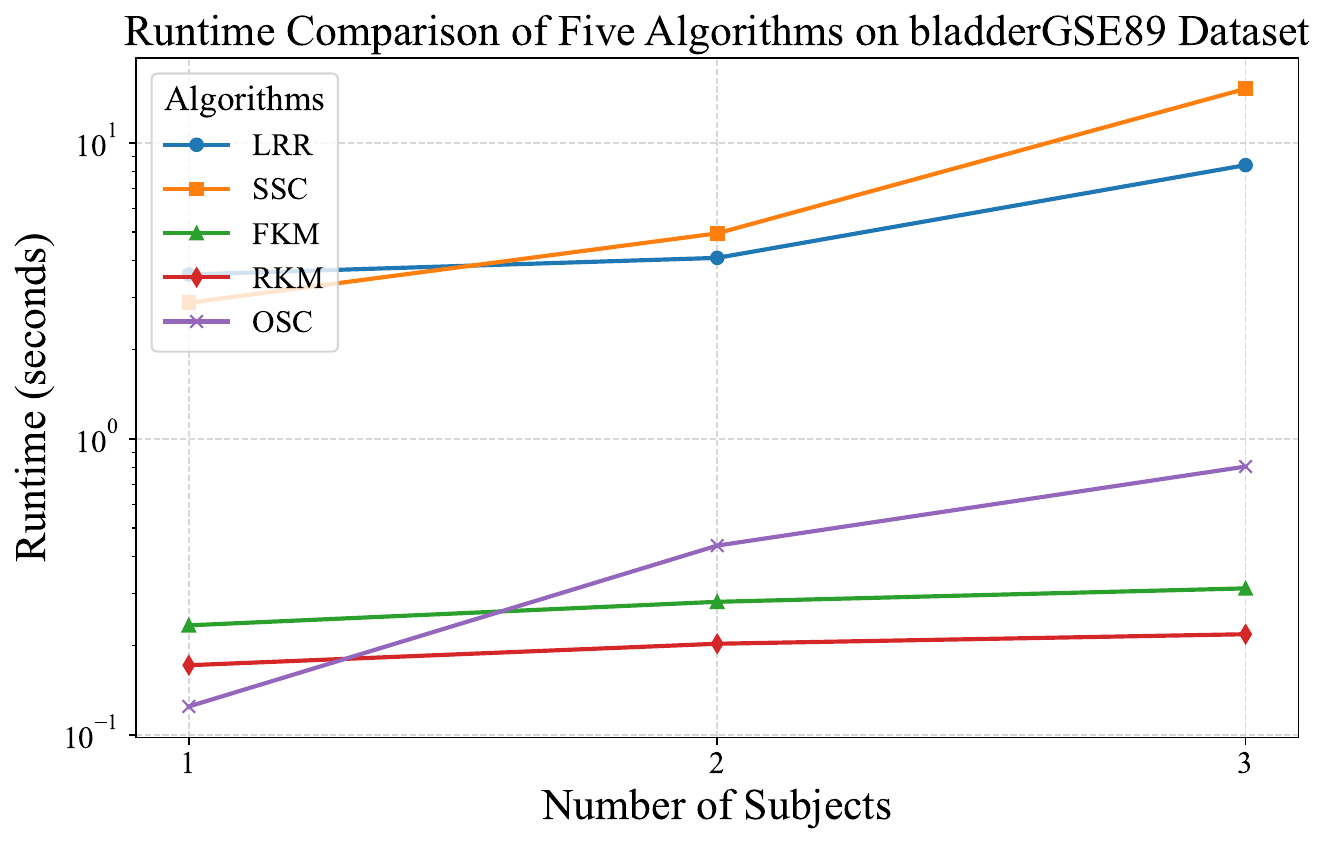}
	\includegraphics[width=0.45\textwidth]{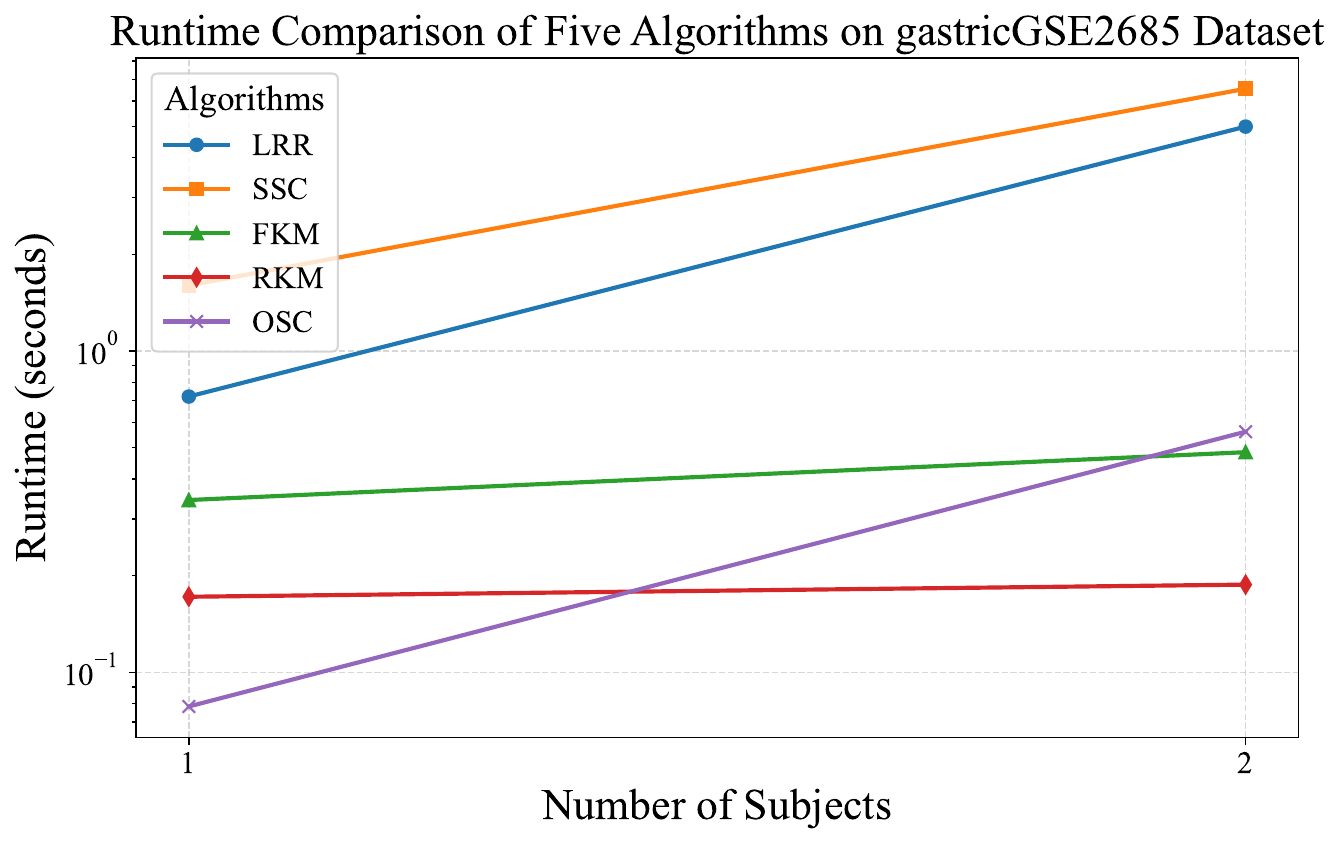}

	\includegraphics[width=0.45\textwidth]{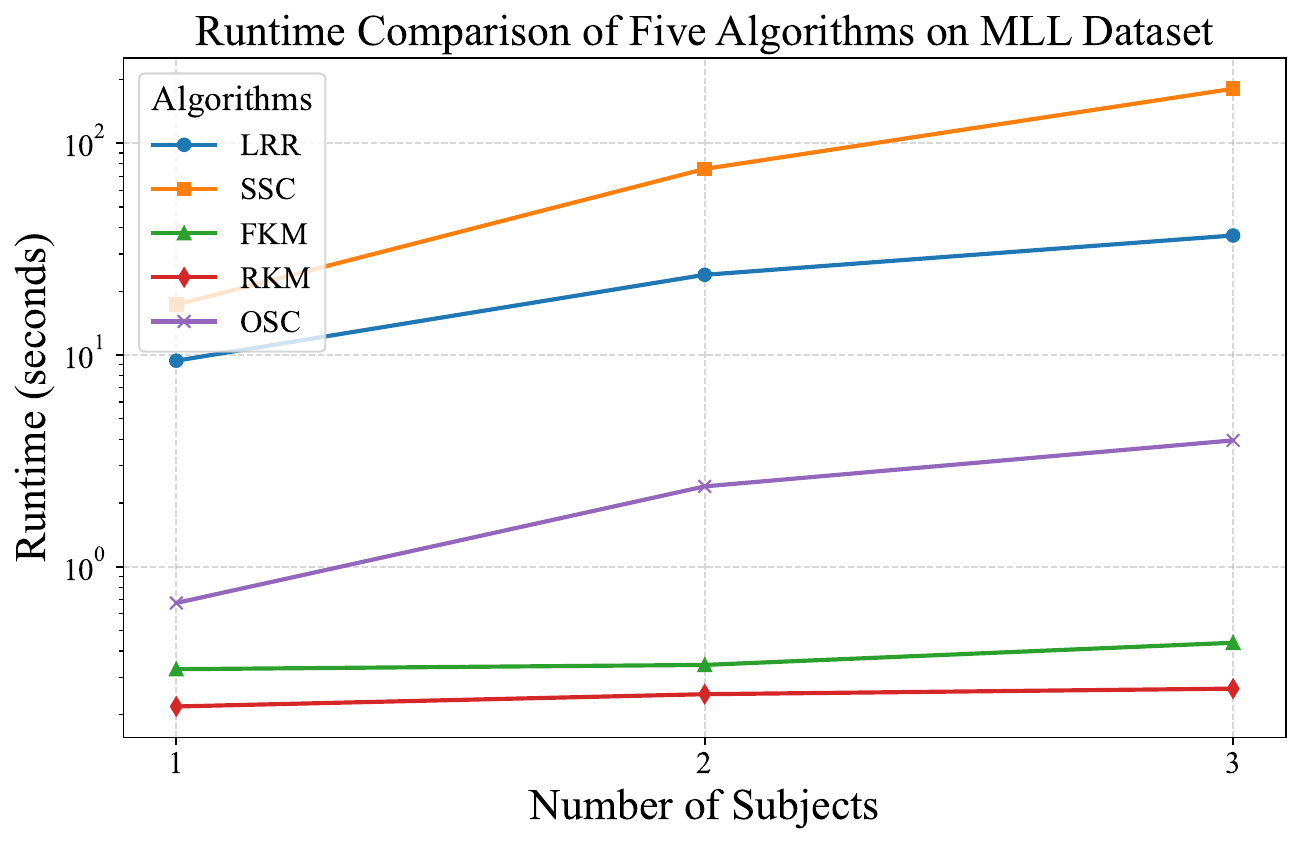}
	\hspace{0.0\linewidth} 
	
	\caption{Algorithm Runtime Comparison for High-Dimensional Datasets }
	\label{fig:runtime_comparison_part}
\end{figure}

The SSC algorithm has a substantially longer processing time than other methods, as the figure illustrates.  The reason behind this is that SSC necessitates calculating the $\ell_1$ norm minimization problem for every data point, usually using iterative algorithms like AML or ADMM, whose computational cost increases dramatically with problem magnitude.  On the other hand, LRR solves a kernel norm minimization issue that can be effectively solved via singular value thresholding and frequently admits closed-form solutions.  Across all datasets, the OSC algorithm continuously maintains low computing time, showing just a slight rise as the number of samples increases. The computational time of the LRR algorithm is less than that of SSC, but it is still greater than that of the suggested approach and exhibits a more noticeable increasing trend with increasing sample size. Although FKM and RKM algorithms often have shorter computation time than SSC, on some datasets they nevertheless show comparatively low computational efficiency.

In conclusion, because of their great algorithmic complexity, SSC and LRR show comparatively low time efficiency. Although RKM and FKM provide faster computation speeds, they perform worse when it comes to clustering. The OSC algorithm allows for fast clustering of jobs across datasets of different sizes, due to its high time efficiency and exceptional performance.

\section{Conclusion }
\label{sec:conclusion}

An OSC technique is proposed in this study.  In order to extract dominant and mutually independent variation patterns from high-dimensional data, its fundamental idea is to build an orthogonal factor model that projects the data onto a set of low-dimensional orthogonal subspaces.  The OSC algorithm successfully finds the underlying structure of high-dimensional data and eliminates unnecessary information while preserving model simplicity and interpretability, according to theoretical analysis and experimental findings.The algorithm finds the directions that contribute the highest variance while guaranteeing their mutual independence by enforcing orthogonal constraints.  This feature shows notable benefits when working with high-dimensional datasets with intricate relationships since it allows OSC to explain the fundamental variation in data during clustering.  OSC achieves competitive clustering performance across several assessment measures, according to experiments conducted on nine publicly available high-dimensional datasets.

However, this study also highlights a number of the OSC algorithm's shortcomings, including the fact that its effectiveness is largely dependent on the precise data variance estimation.  Its sensitivity to anomalous noise is a serious drawback.  The algorithm may mistakenly interpret a noisy sample or outlier that significantly increases variance as a legitimate data structure and create a different subspace for interpretation.  The accuracy of cluster assignments is jeopardized by this overfitting to noise variance, which distorts the underlying subspace structure.  This problem is most noticeable in datasets where there are a lot of outliers or low signal-to-noise ratios.

Future studies will concentrate on the following areas in light of the aforementioned analysis: 

1. \textbf{Robustness Enhancement}: To lessen the effect of outliers on the orthogonal subspace generation process, use robust statistical techniques like sparse modeling. 

2. \textbf{Noise Identification Mechanism}: To discover and manage high-variance noisy points during preprocessing or iterative stages, incorporate an outlier detection module into the OSC framework.


\end{document}